\documentclass{article}

\usepackage{arxiv}

\usepackage[utf8]{inputenc} 
\usepackage[T1]{fontenc}    
\usepackage{hyperref}       
\usepackage{url}            
\usepackage{booktabs}       
\usepackage{amsfonts}       
\usepackage{nicefrac}       
\usepackage{microtype}      

\usepackage{amsmath,amsthm,amssymb,mathtools,algorithm,xcolor}
\usepackage{caption}
\usepackage{subcaption}
\usepackage{enumitem,kantlipsum}
\usepackage{thmtools, thm-restate}

\definecolor{colorY}{rgb}{0.7 , 0.7 , 0.2}

\newcommand{\R}{\mathbb{R}}
\newcommand{\E}{\mathbb{E}}
\newcommand{\C}{\mathcal{C}}
\newcommand{\D}{\mathcal{D}}
\newcommand{\B}{\mathcal{B}}
\newcommand{\A}{\mathcal{A}}
\newcommand{\Vol}{\textsc{Vol}}

\newtheorem{theorem}{Theorem}

\newtheorem{defn}[theorem]{Definition}

\newtheorem{lemma}[theorem]{Lemma}
\newtheorem{corollary}[theorem]{Corollary}

\newtheorem*{remark}{Remark}

\title{Learning piecewise Lipschitz functions in changing environments}

\author{
  Maria-Florina Balcan\\
  School of Computer Science\\
  Carnegie Mellon University\\
  Pittsburgh, PA 15213 \\
  \texttt{ninamf@cs.cmu.edu} \\
  \And
  Travis Dick\\
  Department of Computer Science\\
  Carnegie Mellon University\\
  Pittsburgh, PA 15213 \\
  \texttt{tdick@cs.cmu.edu} \\
  \And
  Dravyansh Sharma\\
  Department of Computer Science\\
  Carnegie Mellon University\\
  Pittsburgh, PA 15213 \\
  \texttt{drasha@cmu.edu} \\
}

\begin{document}
\maketitle

\begin{abstract}
Optimization in the presence of sharp (non-Lipschitz), unpredictable (w.r.t.\ time and amount) changes is a challenging and largely unexplored problem of great significance. We consider the class of piecewise Lipschitz functions, which is the most general online setting considered in the literature for the problem, and arises naturally in various combinatorial algorithm selection problems where utility functions can have sharp discontinuities. The usual performance metric of {\it static} regret minimizes the gap between the payoff accumulated and that of the best fixed point for the entire duration, and thus fails to capture changing environments. Shifting regret is a useful alternative, which allows for up to $s$ environment {\it shifts}. In this work we provide an $O(\sqrt{sdT\log T}+sT^{1-\beta})$ regret bound for $\beta$-dispersed functions, where $\beta$ roughly quantifies the rate at which discontinuities appear in the utility functions in expectation (typically $\beta\ge1/2$ in problems of practical interest \cite{balcan2018dispersion}). We also present a lower bound tight up to sub-logarithmic factors. We further obtain improved bounds when selecting from a small pool of experts. We empirically demonstrate a key application of our algorithms to online clustering problems on popular benchmarks.
\end{abstract}


\section{Introduction}
Online optimization is well-studied in the online learning community \cite{cesa2006prediction,hazan2016introduction}. It consists of a repeated game with $T$ iterations. At iteration $t$, the player chooses a point $\rho_t$ from a compact
decision set $\C\subset\R^d$; after the choice is committed, a bounded utility function $u_t:\C\rightarrow [0,H]$
is revealed. We treat $u_t$ as a reward function to be maximized, although one may also consider minimizing a loss function. The goal of the player is to minimize the regret, defined as the difference between the online cumulative payoff (i.e. $\sum_{t=1}^Tu_t(\rho_t)$) and the cumulative payoff using an optimal offline choice in hindsight. In many real world problems, like online routing \cite{awerbuch2008online,talebi2018stochastic}, detecting spam email/bots \cite{sculley2007relaxed,cormack2008email} and ad/content ranking \cite{wauthier2013efficient,combes2015learning}, it is often inadequate to assume a fixed point will yield good payoff at all times. It is more natural to compute regret against a stronger offline baseline, say one which is allowed to switch the point a few times (say $s$ {\it shifts}), to accommodate events which significantly change the function values for certain time periods. The switching points are neither known in advance nor explicitly stated during the course of the game. This stronger baseline is known as  {\it shifting regret} \cite{herbster1998tracking}.

Shifting regret is a particularly relevant metric for online learning problems in the context of algorithm configuration. This is an important family of non-convex optimization problems where the goal is to decide in a data-driven way what algorithm to use from a large family of algorithms for a given problem domain. In the online setting, one has a configurable algorithm such as an algorithm for clustering data \cite{balcan2017learning}, and must solve a series of related problems, such as clustering news articles each day for a news reader or clustering drugstore sales information to detect disease outbreaks.  For problems of this nature, significant events in the world or changing habits of buyers might require changes in algorithm parameters, and we would like the online algorithms to adapt smoothly.

{\it Our results}: We present the first results for shifting regret for non-convex utility functions which potentially have sharp discontinuities. Restricting attention to specific kinds of decision sets $\C$ and utility function classes yields several important problems. If $\C$ is a convex set and utility functions are concave functions (i.e. corresponding loss functions are convex), we get the Online Convex Optimization (OCO) problem \cite{zinkevich2003online}, which is a generalization of online linear regression \cite{kivinen1997exponentiated} and prediction with expert advice \cite{littlestone1994weighted}. Algorithms with $O(\sqrt{sT\log NT})$ regret are known for the case of $s$ {\it shifts} for prediction with $N$ experts and OCO on the $N$-simplex \cite{herbster1998tracking,cesa2012mirror} using weight-sharing or regularization. We show how to extend the result to arbitrary compact sets of experts, and more general utility functions where convexity can no longer be exploited. Our key insight is to view the regularization as simultaneously inducing multiplicative weights update with restarts matching all possible shifted expert sequences, which allows us to use the {\it dispersion} condition introduced in \cite{balcan2018dispersion}. Related notions like adaptive regret \cite{hazan2007adaptive}, strongly adaptive regret \cite{daniely2015strongly,jun2017improved}, dynamic regret \cite{zinkevich2003online,jadbabaie2015online} and sparse experts setting \cite{bousquet2002tracking} have also been studied for finite experts.

Intuitively, a sequence of piecewise $L$-Lipschitz functions is well-{\it dispersed} if not too many functions are non-Lipschitz in the same region in $\C$. An assumption like this is necessary, since, even for piecewise constant functions, linear regret is unavoidable in the worst case \cite{cohen2017online}. Our shifting regret bounds are $O(\sqrt{sdT\log T}+sT^{1-\beta})$ which imply low regret for sufficiently dispersed (large enough $\beta$) functions. In a large range of applications, one can show $\beta\ge \frac{1}{2}$ \cite{balcan2018dispersion}. This allows us to obtain tight regret bounds modulo sublogarithmic terms, providing a near-optimal characterization of the problem. Our analysis also readily extends to the closely related notion of adaptive regret \cite{hazan2007adaptive}. Note that our setting generalizes the Online Non-Convex Learning (ONCL) problem where all functions are $L$-Lipschitz throughout \cite{maillard2010online,yang2018optimal} for which shifting regret bounds have not been studied. 

We demonstrate the effectiveness of our algorithm in solving the algorithm selection problem for a family of clustering algorithms parameterized by different ways to initialize $k$-means \cite{balcan2018data}. We consider the problem of online clustering, but unlike prior work which studies individual data points arriving in an online fashion \cite{liberty2016algorithm,rakhlin2007online}, we look at complete clustering instances from some distribution(s) presented sequentially. Our experiments provide the first empirical evaluation of online algorithms for piecewise Lipschitz functions --- prior work is limited to theoretical analysis \cite{balcan2018dispersion} or experiments for the batch setting \cite{balcan2018data}. 
Our results also have applications in non-convex online problems like portfolio optimization \cite{merton1976option} and online non-convex SVMs \cite{ertekin2010nonconvex}. More broadly, for applications where one needs to tune hyperparameters that are not {\it nice}, our results imply it is necessary and sufficient to look at dispersion.

{\it Overview}: We formally define the notion of {\it changing environments} in Section \ref{sec:prob}. We then present online algorithms that perform well in these settings in Section \ref{sec:alg}. In Sections \ref{sec:ana} and \ref{sec:sam} we will provide theoretical guarantees of low regret for our algorithms and describe efficient implementations respectively. We will present a near-tight lower bound in the next section. In Section \ref{sec:exp} we demonstrate the effectiveness of our algorithms in algorithm configuration problems for online clustering.

\section{Problem setup}\label{sec:prob}
Consider the following repeated game. At each round $1\le t \le T$ we are required to choose $\rho_t\in\C \subset \R^d$, are presented a piecewise $L$-Lipschitz function $u_t:\C\rightarrow [0,H]$ 
and experience reward $u_t(\rho_t)$. 

In this work we will study $s$-shifted regret and ($m$-sparse, $s$-shifted) regret notions defined below.

\begin{defn}
\label{prob:main}
 The $s$-shifted regret ({\it tracking regret} in \cite{herbster1998tracking}) is given by
\[\E\left[\max_{\substack{\rho_i^*\in\C,\\
t_0=1<t_1\dots<t_s=T+1}}\sum_{i=1}^s\sum_{t=t_{i-1}}^{t_i - 1}(u_t(\rho_i^*)-u_t(\rho_t))\right]\]
\end{defn}


 Note that for the $i$-th phase ($i\in[s]$) given by $[t_{i-1},t_i-1]$, the offline algorithm uses the same point $\rho_i^*$. The usual notion of regret compares the payoff of the online algorithm to the offline strategies that pick a fixed point $\rho^*\in\C$ for all $t\in[T]$ but here we compete against more powerful offline strategies that can use up to $s$ distinct points $\rho_i^*$ 
 by switching the expert $s-1$ times. For $s=1$, we retrieve the standard static regret.

\begin{defn}
\label{prob:sparse}
Extend Definition \ref{prob:main} with an additional constraint on the number of distinct experts used, $\big\lvert\{\rho_i^*\mid1\le i\le s\}\big\rvert\le m$. We call this ($m$-sparse, $s$-shifted) regret \cite{bousquet2002tracking}.

\end{defn}

This restriction makes sense if we think of the adversary as likely to reuse the same experts again, or the changing environment to experience recurring events with similar payoff distributions.

Without further assumptions, no algorithm achieves sublinear regret, even when the payout functions are piecewise constant \cite{cohen2017online}. We will characterize our regret bounds in terms of the {\it dispersion} \cite{balcan2018dispersion} of the utility functions, which roughly says that discontinuities are not too concentrated. Several other restrictions can be seen as a special case \cite{rakhlin2011online,cohen2017online}. 
\begin{defn}\label{def:dis}
The sequence of utility functions $u_1, \dots,u_T$ is $\beta$-{\it dispersed} for the Lipschitz constant $L$ if, for all $T$ and for all $\epsilon\ge T^{-\beta}$, at most
$\Tilde{O}(\epsilon T)$ functions (the soft-O notation suppresses dependence on quantities beside $\epsilon,T$ and $\beta$)
are not $L$-Lipschitz in any ball of size $\epsilon$ contained in $\C$. Further if the utility functions are obtained from some distribution, the random process generating them is said to be $\beta$-dispersed if the above holds in expectation, i.e. if for all $T$ and for all $\epsilon\ge T^{-\beta}$,
\begin{align*}
    \E\left[
\max_{\rho\in\C}\big\lvert
\{ t \mid u_t \text{ is not $L$-Lipschitz in }\B(\rho,\epsilon)\} \big\rvert \right] 
\le  \Tilde{O}(\epsilon T)
\end{align*}
\end{defn}

For {\it static} regret, a continuous version of exponential weight updates gives a tight bound of $\Tilde{O}(\sqrt{dT}+T^{1-\beta})$ \cite{balcan2018dispersion}. They further show that in several cases of practical interest one can prove dispersion with $\beta=1/2$ and the algorithm enjoys $\Tilde{O}(\sqrt{dT})$ regret.  
This algorithm may, however, have $\Omega(T)$ $s$-shifted regret even with a single switch $(s=2)$, and hence is not suited to changing environments (Appendix \ref{appendix:ce}).

\section{Online algorithms with low shifting regret}\label{sec:alg}
In this section we describe online algorithms with good shifting regret, but defer the actual regret analysis to Section \ref{sec:ana}. First we present a discretization based algorithm that simply uses a finite expert algorithm given a discretization of $\C$. This algorithm will give us the reasonable target regret bounds we should shoot for, although the discretization results in exponentially many experts. 

\begin{algorithm}[h]
\caption{Discrete Fixed Share Forecaster}
\label{alg:dsef}
Input: $\beta$, the dispersion parameter
\begin{enumerate}[wide, labelwidth=!, labelindent=0pt]
    \item[1.] Obtain a $T^{-\beta}$-discretization $\D$ of $\C$ (i.e. any $c\in\C$ is within $T^{-\beta}$ of some $d\in\D$)
    \item[2.] Apply an optimal algorithm for finite experts with points in $\D$ as the experts (e.g. fixed share \cite{herbster1998tracking})
\end{enumerate}
\end{algorithm}

We introduce a continuous version of the fixed share algorithm (Algorithm \ref{alg:fsef}). We maintain weights for all points similar to the Exponential Forecaster of \cite{balcan2018dispersion} which updates these weights in proportion to their exponentiated scaled utility $e^{\lambda u_{t}(.)}$ ($\lambda\in(0,1/H]$ is a {\it step size parameter} which controls how aggressively the algorithm updates its weights). The main difference is to update the weights with a mixture of the exponential update and a constant additive boost at all points in some proportion $\alpha$ (the {\it exploration parameter}, optimal value derived in Section \ref{sec:ana}) which remains fixed for the duration of the game. This allows the algorithm to balance exploitation (exponential update assigns high weights to points with high past utility) with exploration, which turns out to be critical for success in changing environments. We will show this algorithm has good $s$-shifted regret in Section \ref{sec:ana}. It also enjoys good adaptive regret \cite{hazan2007adaptive} (see Appendix \ref{appendix:adaptive}).

\begin{algorithm}[h]
\caption{Fixed Share Exponential Forecaster (Fixed Share EF)}
\label{alg:fsef}
Input: step size parameter $\lambda \in (0, 1/H]$, exploration parameter $\alpha\in[0,1]$
\begin{enumerate}[wide, labelwidth=!, labelindent=0pt]
    \item[1.] $w_1(\rho)=1$ for all $\rho\in \C$
    \item[2.] For each $t=1,2,\dots,T$:
    \begin{enumerate}
        \item[i.] $W_t:=\int_{\C}w_t(\rho)d\rho$
        \item[ii.] Sample $\rho$ with probability proportional to $w_t(\rho)$, i.e. with probability
        $p_{t}(\rho)=\frac{w_t(\rho)}{W_t}$
        \item[iii.] Observe $u_t(\cdot)$
        \item[iv.] Let $e_t(\rho)=e^{\lambda u_t(\rho)}w_{t}(\rho)$. For each $\rho\in\C, \text{ set }$
        \begingroup
        \setlength\abovedisplayskip{0pt}
        \begin{equation}\label{eq:wt-fsef}\begin{split}
            \!\,w_{t+1}(\rho)=&(1-\alpha)e_t(\rho)+\frac{\alpha}{\Vol(\C)}\int_{\C}{e_t(\rho)d\rho}
        \end{split}
        \end{equation}
        \endgroup
    \end{enumerate}
\end{enumerate}
\end{algorithm}

Notice that it is not clear how to implement the Algorithm \ref{alg:fsef} from its description. We cannot store all the weights or sample easily since we have uncountably many points $\rho\in\C$.
We will show how to efficiently sample according to $p_t$ without necessarily computing it exactly or storing the exact weights in Section \ref{sec:sam}.

\begin{algorithm}[h!]
\caption{Generalized Share Exponential Forecaster (Generalized Share EF)}
\label{alg:gsef}
Input: step size parameter $\lambda \in (0, 1/H]$, exploration parameter $\alpha\in[0,1]$, discount rate $\gamma\in[0,1]$
\begin{enumerate}[wide, labelwidth=!, labelindent=0pt]
    \item[1.] $w_1(\rho)=1$ for all $\rho\in \C$
    \item[2.] For each $t=1,2,\dots,T$:
    \begin{enumerate}
        \item[i.] $W_t:=\int_{\C}w_t(\rho)d\rho$
        \item[ii.] Sample $\rho$ with probability proportional to $w_t(\rho)$, i.e. with probability
        $p_{t}(\rho)=\frac{w_t(\rho)}{W_t}$
        \item[iii.] Let $e_t(\rho)=e^{\lambda u_t(\rho)}w_{t}(\rho)$ and $\beta_{i,t} = \frac{e^{-\gamma(t-i)}}{\sum_{j=1}^te^{-\gamma(t-j)}}$. For each $\rho\in\C$, set $$w_{t+1}(\rho)=(1-\alpha)e_t(\rho)+\alpha\left(\int_{\C}{e_t(\rho)d\rho}\right)\sum_{i=1}^t{\beta_{i,t}p_i(\rho)}$$
    \end{enumerate}
\end{enumerate}
\end{algorithm}

As it turns out adding equal weights to all points for exploration does not allow us to exploit recurring environments of the ($m$-sparse, $s$-shifted) setting very well. To overcome this, we replace the uniform update with a prior consisting of a weighted mixture of all the previous probability distributions used for sampling (Algorithm \ref{alg:gsef}). Notice that this includes uniformly random exploration as the first probability distribution $p_1(\cdot)$ is uniformly random, but the weight on this distribution decreases exponentially with time according to {\it discount rate} $\gamma$ (more precisely, it decays by a factor $e^{-\gamma}$ with each time step). While exploration in Algorithm \ref{alg:fsef} is limited to starting afresh, here it includes partial resets to explore again from all past states, with an exponentially discounted rate (cf. Theorems \ref{thm:ub}, \ref{thm:ub-sparse}).

\section{Analysis of algorithms}\label{sec:ana}
We will now analyse the algorithms in Section \ref{sec:alg}. At a high level, the algorithms have been designed to ensure that the optimal solution, and its neighborhood, in hindsight have a large total density. We achieve this by carefully setting the parameters, in particular the {\it exploration parameter} which controls the rate at which we allow our confidence on {\it good} experts to change. Lipschitzness and dispersion are then used to ensure that solutions sufficiently close to the optimum are also good on average.

\subsection{Regret bounds} In the remainder of this section we will have the
following setting. We assume the utility functions $u_t:\C\rightarrow [0,H],t\in[T]$ are piecewise $L$-Lipschitz and $\beta$-dispersed (definition \ref{def:dis}), where $\C\subset \R^d$ is contained in a ball of radius $R$.
\begin{restatable}{theorem}{disc}
\label{lem:disc}
 Let $R^{\text{finite}}(T,s,N)$ denote the $s$-shifted regret for the finite experts problem on $N$ experts, for the algorithm  used in step 2 of Algorithm \ref{alg:dsef}. Then Algorithm  \ref{alg:dsef} enjoys $s$-shifted regret $R^{\C}(T,s)$ which satisfies
\[R^{\C}(T,s)\le R^{\text{finite}}\left(T,s,\left(3RT^{\beta}\right)^d\right)+(sH+L)O(T^{1-\beta}).\]
\end{restatable}
The proof of Theorem \ref{lem:disc} is straightforward using the definition of dispersion and is deferred to Appendix \ref{appendix:disc}. This gives us the following target bound for our more efficient algorithms. 
\begin{corollary}
The $s$-shifted regret of Algorithm \ref{alg:dsef} is $O(H\sqrt{sT(d\log (RT^{\beta})+\log(T/s))}+(sH + L)T^{1-\beta})$.
\end{corollary}
\begin{proof}
There are known algorithms e.g. Fixed-Share (\cite{herbster1998tracking}) which obtain $R^{\text{finite}}(T,s,N) \le O(\sqrt{sT\log (NT/s)})$.
Applying Theorem \ref{lem:disc} gives the desired upper bound.
\end{proof}
Under the same conditions, we will show the following bounds for our algorithms. In the following statements, we give approximate values for the parameters $\alpha,\gamma$ and $\lambda$ under the assumptions $m\ll s, s\ll T$. See proofs in Appendix \ref{appendix:analysis} for more precise values.
\begin{restatable}{theorem}{thmub}
\label{thm:ub}
The $s$-shifted regret of Algorithm \ref{alg:fsef}  with $\alpha=s/T$  and $\lambda=\sqrt{s(d\log(RT^{\beta})+\log(T/s))/T}/H$  is 
$O(H\sqrt{sT(d\log (RT^{\beta})+\log(T/s))}+(sH + L)T^{1-\beta})$.

\end{restatable}
\begin{remark}
The algorithms assume knowledge of $s/T$, the average number of shifts per time. For unknown $s$, the strongly adaptive algorithms of \cite{daniely2015strongly,jun2017improved} can be used with the same meta-algorithms and substituting continuous exponential forecasters as black-box algorithms.
\end{remark}
Similarly for Algorithm \ref{alg:gsef} we can show low ($m$-sparse, $s$-shifted) regret as well. (In particular this implies $s$-shifted regret almost as good as Algorithm \ref{alg:fsef}.)
\begin{restatable}{theorem}{thmubsparse}
\label{thm:ub-sparse}
The ($m$-sparse, $s$-shifted) regret of Algorithm \ref{alg:gsef} is  $O(H\sqrt{T(md\log (RT^{\beta})+s\log(mT/s))}+(mH + L)T^{1-\beta})$
for  $\alpha=s/T$, $\gamma=s/mT$ and $\lambda=\sqrt{(md\log(RT^{\beta})+s\log(T/s))/T}/H$.
\end{restatable}

\subsection{Proof sketch and insights}
We start with some observations about the weights $W_t$ in Algorithm \ref{alg:fsef}.
\begin{lemma}[Algorithm \ref{alg:fsef}]
\label{lem:wt-basic}
$W_{t+1} = \int_{\C}e^{\lambda u_t(\rho)}w_{t}(\rho)d\rho$.
\end{lemma}
The update rule (\ref{eq:wt-fsef}) had the uniform exploration term scaled just appropriately so this relation is satisfied. We will now relate $W_t$ with weights resulting from pure exponential updates, i.e. $\alpha=0$ in Algorithm \ref{alg:fsef} (also the Exponential Forecaster algorithm of \cite{balcan2018dispersion}). The following definition corresponds to weights for running Exponential Forecaster starting at some time $\tau$.
\begin{defn}\label{def:w-tilde} For any $\rho\in\C$ and $\tau\le\tau'\in[T]$ define $\Tilde{w}(\rho;\tau,\tau')$ to be the weight of expert $\rho$, and $\Tilde{W}(\tau,\tau')$ to be the normalizing constant, if we ran the  Exponential Forecaster of \cite{balcan2018dispersion} starting from time $\tau$ up till time $\tau'$, i.e. $\Tilde{w}(\rho;\tau,\tau'):=e^{\lambda \sum_{t=\tau}^{\tau'-1}u_t(\rho)}$ and $\Tilde{W}(\tau,\tau'):=\int_{\C}\Tilde{w}(\rho;\tau,\tau')d\rho$.
\end{defn}
We consider Algorithm \ref{alg:rref} obtained by a slight modification in the update rule (\ref{eq:wt-fsef}) of Fixed Share EF (Algorithm \ref{alg:fsef}) which makes it easier to analyze. Essentially we replace the deterministic $\alpha$-mixture by a randomized one, so at each turn we either explicitly {\it restart} with probability $\alpha$ by putting the same weight on each point, or else apply the exponential update. We note that Algorithm \ref{alg:rref} is introduced to simplify the proof of Theorem \ref{thm:ub}, and in particular does {\it not} result in low regret itself. The issue is that even though the weights are correct in expectation (Lemma \ref{lem:fsef-rref}), their ratio (probability $p_t(\rho)$) is not. In particular, the optimal parameter value of $\alpha$ for Fixed Share EF allows the possibility of pure exponential updates over a long period of time with a constant probability in Algorithm \ref{alg:rref}, which implies linear regret (see Appendix \ref{appendix:ce}, Theorem \ref{thm:ce}). This also makes the implementation of Fixed Share EF somewhat trickier (Section \ref{sec:sam}).

\begin{algorithm}[h]
\caption{Random Restarts Exponential Forecaster (Random Restarts EF)}
\label{alg:rref}
Input: step size parameter $\lambda \in (0, 1/H]$, exploration parameter $\alpha\in[0,1]$
\begin{enumerate}[wide, labelwidth=!, labelindent=0pt]
    \item[1.] $\hat{w}_1(\rho)=1$ for all $\rho\in \C$
    \item[2.] For each $t=1,2,\dots,T$:
    \begin{enumerate}
        \item[i.] $\hat{W}_t:=\int_{\C}\hat{w}_t(\rho)d\rho$
        \item[ii.] Sample $\rho$ with probability  proportional to $\hat{w}_t(\rho)$, i.e. with probability
        $p_{t}(\rho)=\frac{\hat{w}_t(\rho)}{\hat{W}_t}$
        \item[iii.] Sample $z_t$ uniformly in $[0,1]$ and set
        \begin{equation*}
        \hat{w}_{t+1}(\rho)=\begin{cases*}
        e^{\lambda u_t(\rho)}\hat{w}_{t}(\rho) &if $z_t<1-\alpha$\\
        \frac{\int_{\C}{e^{\lambda u_t(\rho)}\hat{w}_{t}(\rho)d\rho}}{\Vol(\C)} &otherwise\end{cases*}
        \end{equation*}
    \end{enumerate}
\end{enumerate}
\end{algorithm}

The expected weights of Algorithm \ref{alg:rref} (over the coin flips used in weight setting) are the same as the actual weights of Algorithm \ref{alg:fsef} (proof in Appendix \ref{appendix:analysis}).

\begin{restatable}{lemma}{lemfsefrref}\textbf{(Algorithm \ref{alg:fsef})}
\label{lem:fsef-rref}
For each $t\in[T]$, $w_{t}(\rho)=\E[\hat{w}_t(\rho)]$ and $W_{t}=\E[\hat{W}_t]$, where the expectations are over random restarts $\mathbf{z}_t=\{z_1,\dots,z_{t-1}\}$.
\end{restatable}

The next lemma provides intuition for looking at our algorithm as a weighted superposition of several exponential update subsequences with restarts. This novel insight establishes a tight connection between the algorithms and is crucial for our analysis.
\begin{restatable}{lemma}{lemwtpartition}\textbf{(Algorithm \ref{alg:fsef})}
$W_{T+1}$ equals the sum
\label{lem:wt-partition}
\begin{align*}\sum_{s\in [T]}\sum_{t_0=1<t_1\dots<t_s=T+1}\frac{\alpha^{s-1}(1-\alpha)^{T-s}}{\Vol(\C)^{s-1}}\prod_{i=1}^s\Tilde{W}(t_{i-1},t_i)\end{align*}
\end{restatable}
\begin{proof}[Proof Sketch]
Each term corresponds to the weight when we pick a number $s\in[T]$ for the number of times we start afresh with a uniformly random point $\rho$ at times $\mathbf{t}_s=\{t_1,\dots,t_{s-1}\}$ and do the regular exponential weighted forecaster in the intermediate periods. We have a weighted sum over all these terms with a factor $\alpha/\Vol(\C)$ for each time we restart and $(1-\alpha)$ for each time we continue with the Exponential Forecaster.
\end{proof}

We will now prove Theorem \ref{thm:ub}. The main idea is to show that the normalized exploration helps the total weights to provide a lower bound for the algorithm payoff. Also the total weights are competitive against the optimal payoff as they contain the exponential updates with the optimal set of switching points in Lemma \ref{lem:wt-partition} with a sufficiently large (probability) coefficient.

\begin{proof}[Proof sketch of Theorem \ref{thm:ub}]
We provide an upper and lower bound to $\frac{W_{T+1}}{W_{1}}$. The upper bound uses Lemma \ref{lem:wt-basic} and helps us lower bound the performance of the algorithm (see Appendix \ref{appendix:analysis}) as
\begin{equation}\label{eq:wt-ub}\frac{W_{T+1}}{W_{1}}\le \exp\left(\frac{P(\mathcal{A})(e^{H\lambda}-1)}{H}\right)\end{equation}
where $P(\mathcal{A})$ is the expected total payoff for Algorithm \ref{alg:fsef}.
We now upper bound the optimal payoff $OPT$ by providing a lower bound for $\frac{W_{T+1}}{W_{1}}$. By Lemma \ref{lem:wt-partition} we have
\begin{align*}
W_{T+1}\ge  \frac{\alpha^{s-1}(1-\alpha)^{T-s}}{\Vol(\C)^{s-1}}\prod_{i=1}^s\Tilde{W}(t^*_{i-1},t^*_i)
\label{eq:wt-lb}
\end{align*}
by dropping all terms save those that {\it restart} exactly at the OPT expert switches $t^*_{0:s}$. Now using $\beta$-dispersion we can show (full proof in Appendix \ref{appendix:analysis}) \begin{align*}\frac{W_{T+1}}{W_1}\ge&\frac{\alpha^{s-1}(1-\alpha)^{T-s}}{(RT^{\beta})^{sd}}e^{\lambda\left(OPT - (sH + L)O(T^{1-\beta})\right)}\end{align*}
Putting together with the upper bound (\ref{eq:wt-ub}), rearranging and optimizing the difference for $\alpha$ and $\lambda$ concludes the proof.  (See Appendix \ref{appendix:analysis} for a full proof.)
% ...
\end{proof}

We now analyze Algorithm \ref{alg:gsef} for the sparse experts setting. We can adapt proofs of Lemmas \ref{lem:wt-basic} and \ref{lem:wt-partition} to easily establish Lemmas \ref{lem:wt-basic-gsef} and \ref{lem:wt-partition-gsef}.
\begin{lemma}[Algorithm \ref{alg:gsef}]
\label{lem:wt-basic-gsef}

$W_{t+1} = \int_{\C}e^{\lambda u_t(\rho)}w_{t}(\rho)d\rho$.
\end{lemma}
\begin{lemma}
Let $\pi_t(\rho)=\sum_{i=1}^{t-1}\beta_{i,t}p_i(\rho)$. For Algorithm \ref{alg:gsef}, $W_{T+1}$ can be shown to be equal to the sum
\label{lem:wt-partition-gsef}
\begin{align*}
    \sum_{s\in [T]}\sum_{t_0=1<\dots t_s=T+1}\alpha^{s-1}(1-\alpha)^{T-s}\prod_{i=1}^s\Tilde{W}(\pi_{t_{i-1}};t_{i-1},t_i)
\end{align*}
where 
$\Tilde{W}(p;\tau,\tau'):=\int_{\C}p(
\rho)\Tilde{w}(\rho;\tau,\tau')d\rho$.
\end{lemma}
\begin{corollary}
\label{cor:wt-partition-gsef}

$W_{T}\ge \alpha(1-\alpha)^{T-t}\Tilde{W}(\pi_{t};t,T)W_t$, for all $t<T$.
\end{corollary}
\begin{proof}
Consider the probability of last {\it reset} (setting $w_{t}(\rho)=W_t\pi_{t}(\rho)$) at time $t$ when computing $W_{T+1}$ as the expected weight of a random restart version which matches Algorithm \ref{alg:gsef} till time $t$.
\end{proof}

Now to prove Theorem \ref{thm:ub-sparse}, we show that the total weight is competitive with running exponential updates on all partitions (in particular the optimal partition) of $[T]$ into $m$ subsets with $s$ switches, intuitively the property of restarting exploration from all past points crucially allows us to {\it jump} across intervals where a given expert was inactive (or bad).
\begin{proof}[Proof sketch of Theorem \ref{thm:ub-sparse}]
We provide an upper and lower bound to $\frac{W_{T+1}}{W_{1}}$ similar to Theorem \ref{thm:ub}. Using Lemma \ref{lem:wt-basic-gsef} we can show that inequality \ref{eq:wt-ub} holds here as well. 
By Corollary \ref{cor:wt-partition-gsef} and Lemma \ref{lem:sparse} (which relates $\pi_t(.)$ to past weights, proved in Appendix \ref{appendix:analysis}), and $\beta$-dispersion we can show a better lower bound.
\begin{equation}\label{eq:wt-lb-gsef}
\begin{split}\frac{W_{T+1}}{W_1}\ge&
\frac{\alpha^{s}(1-\alpha)^{T}(1-e^{-\gamma})^s}{(e^{-\gamma}+\alpha(1-e^{-\gamma}))^{-mT}\left(RT^{\beta}\right)^{md}}\exp\left(\lambda\left(OPT - (mH + L)O(T^{1-\beta})\right)\right)
\end{split}\end{equation}
Putting together the lower and upper bounds, rearranging and optimizing for $\gamma,\alpha,\lambda$ concludes the proof.
\end{proof}

\section{Efficient implementation of algorithms}\label{sec:sam}
In this section we show that the Fixed Share Exponential Forecaster algorithm (Algorithm \ref{alg:fsef}) can be implemented efficiently when $u_t$'s are piecewise concave (dimishing returns). In particular we overcome the need to explicitly compute and update $w_t(\rho)$ (there are uncountably infinite $\rho$ in $\C$) by showing that we can sample the points according to $p_t(\rho)$ directly.\\
The high-level strategy is to show (Lemma \ref{lem:pt-recursion}) that $p_t(\rho)$ is a mixture of $t$ distributions which are Exponential Forecaster distributions from \cite{balcan2018dispersion} i.e. $\Tilde{p}_i(\rho):=\frac{\Tilde{w}(\rho;i,t)}{\Tilde{W}(i,t)}$ for each $1\le i\le t$, with proportions $C_{i,t}$. As shown in \cite{balcan2018dispersion} these distributions can be approximately sampled from (exactly in the one-dimensional case, $\C\subset \R$), summarized below as Algorithm BDV-18. We need to sample from one of these $t$ distributions with probability $C_{t,i}$ to get the distribution $p_t$, and we can approximate these coefficients efficently (or compute exactly in one-dimensional case). The rest of the section discusses how to do these approximations efficiently, and with small extra expected regret. Asymptotically we get the same bound as the exact algorithm. (Formal proofs in Appendix \ref{appendix:efficient}).\\\\{\bf Algorithm BDV-18}: Simply {\it integrate} pieces of the exponentiated utility function, pick a piece with probability proportional to its integral, and {\it sample from} that piece. \cite{lovasz2006fast} show how to efficient {\it sample from} and {\it integrate} logconcave distributions. See \cite{balcan2018dispersion} for more details.\\\\The coefficients have a simple form in terms of normalizing constants $W_t$'s of the rounds so far, so we first express $W_{t+1}$ in terms of $W_{t}$'s from previous rounds and some $\Tilde{W}(i,j)$'s.
\begin{restatable}{lemma}{lemwtrecursion}
\label{lem:wt-recursion} In Algorithm \ref{alg:fsef}, for $t\ge 1$,
\begin{align*}
    W_{t+1} = &(1-\alpha)^{t-1}\Tilde{W}(1,t+1)+\\
    &\frac{\alpha}{\Vol(\C)}\sum_{i=2}^{t}\bigg[(1-\alpha)^{t-i}W_i\Tilde{W}(i,t+1)\bigg]
\end{align*}
\end{restatable}
As indicated above, $p_t(\rho)$ is a mixture of $t$ distributions.
\begin{restatable}{lemma}{lemptrecursion}
\label{lem:pt-recursion}
In Algorithm \ref{alg:fsef}, for $t\ge 1$,
$p_t(\rho)=\sum_{i=1}^{t}C_{t,i}\frac{\Tilde{w}(\rho;i,t)}{\Tilde{W}(i,t)}$. 
The coefficients $C_{t,i}$ are given by
\begin{equation*}
    C_{t,i} =
    \begin{cases*}
        1&$i=t=1$\\
        \alpha&$i=t>1$\\
        (1-\alpha)\frac{W_{t-1}}{W_{t}}\frac{\Tilde{W}(i,t)}{\Tilde{W}(i,t-1)}C_{t-1,i}&$i<t$
    \end{cases*}
\end{equation*}
and $(C_{t,1},\dots,C_{t,t})$ lies on the probability simplex $\Delta^{t-1}$.  
\end{restatable}
The observations above allow us to write the  algorithms for efficiently implementing Fixed Share EF, for which we obtain formal guarantees in Theorem \ref{thm:sampling}. We present an approximate algorithm (Algorithm \ref{alg:fsef-d}) with the same expected regret as in Theorem \ref{thm:ub} (and also present an exact algorithm, Algorithm \ref{alg:fsef-1d} in Appendix \ref{appendix:efficient}, for $d=1$). We say Algorithm \ref{alg:fsef-d} gives a $(\eta,\zeta)$ estimate of Algorithm \ref{alg:fsef}, i.e. with probability at least $1-\zeta$, its expected payoff is within a factor of $e^\eta$ of that of Algorithm \ref{alg:fsef}.
\begin{algorithm}
\caption{Fixed Share Exponential Forecaster - efficient approximate implementation}
\label{alg:fsef-d}
{\bf Input:} approximation parameter $\eta\in (0, 1)$, confidence parameter $\zeta \in (0, 1)$
\begin{enumerate}[wide, labelwidth=!, labelindent=0pt]
    \item[1.] $W_1=\Vol(\C)$
    \item[2.] For each $t=1,2,\dots,T$:
    \begin{enumerate}
        \item[i.] Estimate $C_{t,j}$ using Lemma \ref{lem:pt-recursion} for each $1\le j \le t$.
        \item[ii.] Sample $i$ with probability $C_{t,i}$.
        \item[iii.] Sample $\rho$ with probability approximately proportional to $\Tilde{w}(\rho;i,t)$ by running Algorithm BDV-18 with approximation-confidence parameters $(\eta/3,\zeta/2)$.
        \item[iv.] Estimate $W_{t+1}$ using Lemma \ref{lem:wt-recursion}. Algorithm BDV-18 to get $(\eta/6T,\eta/2T^2)$ estimates for all $\Tilde{W}(\tau,\tau')$  and memoize values of $W_{i},i\le t$.
    \end{enumerate}
\end{enumerate}
\end{algorithm}

\begin{restatable}{theorem}{thmsampling}
\label{thm:sampling}
If utility functions are piecewise concave and $L$-Lipschitz, we can approximately sample a point $\rho$ with probability $p_{t+1}(\rho)$ in time $\Tilde{O}(Kd^4T^4)$ for approximation parameters $\eta=\zeta=1/\sqrt{T}$ and $\lambda=\sqrt{s(d\ln(RT^{\beta})+\ln(T/s))/T}/H$ and enjoy the same regret bound as the exact algorithm. ($K$ is the number of discontinuities in $u_t$'s).
\end{restatable}

Note that in this section we concerned ourselves with developing a $poly(d,T)$ algorithm. For special cases of practical interest, like one-dimensional piecewise constant functions, we can implement much faster $O(K\log KT)$ algorithms as noted in Section \ref{sec:exp}.

\section{Lower bounds}\label{sec:lb}

We prove our lower bound for $\C=[0,1]$ and $H=1$. Also we will consider functions which are $\beta$-dispersed and 0-Lipschitz (piecewise constant). For such utility functions $u_1,\dots,u_T$ we have shown in Section \ref{sec:ana} that the $s$-shifted regret is $O(\sqrt{sT\log T}+sT^{1-\beta})$. Here we will establish a lower bound of $\Omega(\sqrt{sT}+sT^{1-\beta})$.

We show a range of values of $s,\beta$ where the stated lower bound is achieved. For $s=1$, this improves over the lower bound construction of \cite{balcan2018dispersion} where the lower bound is shown only for $\beta=1/2$. In particular our results establish an almost tight characterization of static and dynamic regret under dispersion.
\begin{restatable}{theorem}{thmlb}\label{thm:lb}
For each $\beta>\frac{\log 3s}{\log T}$, there exist utility functions $u_1,\dots,u_T : [0,1]\rightarrow [0,1]$ which are $\beta$-dispersed, and the $s$-shifted regret of any online algorithm is $\Omega(\sqrt{sT}+sT^{1-\beta})$.
\end{restatable}
\vspace*{-3mm}
\begin{proof} We perform the construction in $\Theta(s)$ phases, each phase accumulating $\Omega(\sqrt{T/s}+T^{1-\beta})$ regret, yielding the desired lower bound.\\\\
Let $I_1=[0,1]$. In the first phase, for the first $\frac{T-3sT^{1-\beta}}{s}$ functions we have a single discontinuity in the interval $\left(\frac{1}{2}\left(1-\frac{1}{3s}\right),\frac{1}{2}\left(1+\frac{1}{3s}\right)\right)\subseteq(\frac{1}{3},\frac{2}{3})$. The functions have payoff 1 before or after (with probability $1/2$ each) their discontinuity point, and zero elsewhere. We introduce $3T^{1-\beta}$ functions each for the same discontinuity point, and set the discontinuity points $T^{-\beta}$ apart for $\beta$-dispersion. This gives us $\frac{1/3s}{T^{-\beta}}-1$ potential points inside $[\frac{1}{3},\frac{2}{3}]$, so we can support $3T^{1-\beta}\left(\frac{1/3s}{T^{-\beta}}-1\right)=\frac{T}{s}-3T^{1-\beta}$ such functions ($\frac{T}{s}-3T^{1-\beta}>0$ since $\beta>\frac{\log 3s}{\log T}$). By Lemma \ref{lem:2-exp-lb} (Appendix \ref{appendix:lb}) we accumulate   $\Omega(\sqrt{\frac{T-3sT^{1-\beta}}{s}})=\Omega(\sqrt{T/s})$ regret for this part of the phase in expectation.\\\\
Let $I_1'$ be the interval from among $[0,\frac{1}{2}\left(1-\frac{1}{3s}\right)]$ and $[\frac{1}{2}\left(1+\frac{1}{3s}\right),1]$ with more payoff in the phase so far. The next function has payoff 1 only at first or second half of $I_1'$ (with probability $1/2$) and zero everywhere else. Any algorithm accumulates expected regret $1/2$ on this round. We repeat this in successively halved intervals. $\beta$-dispersion is satisfied since we use only $\Theta(T^{1-\beta})$ functions in the interval $I'$ of size greater than $1/3$, and we accumulate an additional $\Omega(T^{1-\beta})$ regret. Notice there is a fixed point used by the optimal adversary for this phase.\\\\
Finally we repeat the construction inside the largest interval with no discontinuities at the end of the last phase for the next phase. Note that at the $i$-th phase the interval size will be $\Theta(\frac{1}{i})$. Indeed at the end of the first round we have {\it unused} intervals of size $\frac{1}{2}\left(1-\frac{1}{3s}\right),\frac{1}{4}\left(1-\frac{1}{3s}\right),\frac{1}{8}\left(1-\frac{1}{3s}\right),\dots$ At the $i=2^j$-th phase, we'll be repeating inside an interval of size $\frac{1}{2^{j+1}}\left(1-\frac{1}{3s}\right)=\Theta(\frac{1}{i})$. This allows us to run $\Theta(s)$ phases and get the desired lower bound (intervals must be of size at least $\frac{1}{s}$ to support the construction). 
\end{proof}

\section{Experiments}\label{sec:exp}

The simplest demonstration of significance of our algorithm in a changing environment is to consider the 2-shifted regret when a single expert shift occurs. We consider an artifical online optimization problem first, and will then look at applications to online clustering. Let $\C=[0,1]$. Define utility functions
\begin{align*}
    u^{(0)}(\rho)=
    \begin{cases*}
    1&if $\rho < \frac{1}{2}$\\
    0&if $\rho \ge \frac{1}{2}$
    \end{cases*}\text{ and }u^{(1)}(\rho)=
    \begin{cases*}
    0&if $\rho < \frac{1}{2}$\\
    1&if $\rho \ge \frac{1}{2}$
    \end{cases*}
\end{align*}

\begin{figure}[H]
    \centering
  \vspace*{-5mm}
         \includegraphics[width=0.35\textwidth]{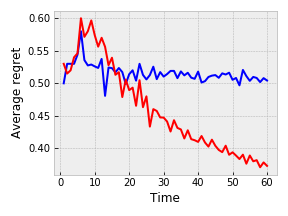}
  \vspace*{-5mm}
  \caption{Average $2$-shifted regret vs game duration $T$ for a game with single expert shift. Color scheme: \textcolor{blue}{\bf Exponential Forecaster}, \textcolor{red}{\bf Fixed Share EF}}
\label{fig:01}
\end{figure}
Now consider the instance where $u^{(0)}(\rho)$ is presented for the first $T/2$ rounds and $u^{(1)}(\rho)$ is presented for the remaining rounds. We observe constant average regret for the Exponential Forecaster algorithm, while Fixed Share regret decays as $O(1/\sqrt{T})$ (Figure \ref{fig:01}). While the example is simple and artificial, it qualitatively captures why Fixed Share dominates Exponential Forescaster here --- because the best expert changes and the old expert is no longer competitive. (cf. Appendix \ref{appendix:ce})

$k$-means$++$ is a celebrated algorithm \cite{arthur2007k} which shows the importance of initial seed centers in clustering using the $k$-means algorithm (also called Llyod's method). Balcan et al. \cite{balcan2018data} generalize it to $(\Bar{\alpha},2)$-Lloyds$++$-clustering, which interpolates between random initial seeds (vanilla $k$-means, $\Bar{\alpha}=0$), $k$-means$++$ ($\Bar{\alpha}=2$) and farthest-first traveral ($\Bar{\alpha}=\infty$) \cite{gonzalez1985clustering,dasgupta2005performance} using a single parameter $\Bar{\alpha}$. The clustering objective (we use
the Hamming distance to the optimal clustering, i.e. the fraction of points assigned to different
clusters by the algorithm and the target clustering) is a piecewise constant function of $\Bar{\alpha}$, and the best clustering may be obtained for a value of $\Bar{\alpha}$ specific to a given problem domain. In an online problem, where clustering instances arrive in a sequential fashion, determining good values of $\Bar{\alpha}$ becomes an online optimization problem on piecewise Lipshitz functions. Furthermore the functions are $\beta$-dispersed for $\beta=1/2$ \cite{balcan2018data}.

We perform our evaluation on four benchmark datasets to cover a range of examples-set sizes, $N$ and number of clusters, $k$: {\it MNIST}, $28 \times 28$ binary images of handwritten digits with 60,000 training examples for 10 classes \cite{deng2012mnist}; {\it Omniglot}, $105 \times 105$ binary images of handwritten characters across 30 alphabets with 19,280 examples \cite{lake2015human}; {\it Omniglot\_small\_1}, a minimal Omniglot split with only 5 alphabets and 2720 examples.

We consider a sequence of clustering instances drawn from the four datasets and compare our algorithms Fixed Share EF (Algorithm \ref{alg:fsef}) and Generalized Share EF (Algorithm \ref{alg:gsef}) with the Exponential Forecaster algorithm of \cite{balcan2018dispersion}. At each time $t \le T \le 60$ we sample a subset of the dataset of size $100$. For each $T$, we take uniformly random points from half the classes (even class labels) at times $t=1,\dots,T/2$ and from the remaining classes (odd class labels) at $T/2 < t \le T$. We determine the hamming cost of $(\Bar{\alpha}, 2)$-Lloyds$++$-clustering for $\alpha\in\C=[0,10]$ which is used as the piecewise constant loss function (or payoff is the fraction of points assigned correctly) for the online optimization game. Notice the Lipschitz constant $L=0$ since we have piecewise constant utility, and utility function values lie in $[0,1]$. We set exploration parameter $\alpha=1/T$ and decay parameter $\gamma=1/T$ in our algorithms. 
We plot average $2$-shifted regret until time $T$ (i.e. $R_T/T$) and take average over 20 runs to get smooth curves. (Figure \ref{fig:1}). Unlike Figure \ref{fig:01}, the optimal clustering parameters before the shift might be relatively competitive to new optimal parameters. So the Exponential Forecaster performance is not terrible, although our algorithms still outperform it noticeably. 

\begin{figure*}[ht]
    \centering
    \begin{subfigure}[b]{0.32\textwidth}
    \centering
         \includegraphics[width=\textwidth]{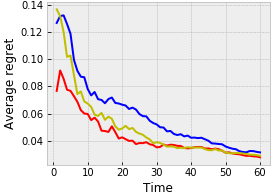}
    \caption{MNIST}
  \end{subfigure}
    \begin{subfigure}[b]{0.32\textwidth}
    \centering
         \includegraphics[width=\textwidth]{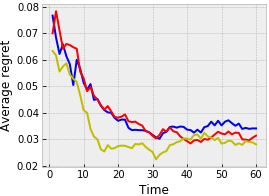}
    \caption{Omniglot\_small\_1}
  \end{subfigure}
    \begin{subfigure}[b]{0.32\textwidth}
    \centering
         \includegraphics[width=\textwidth]{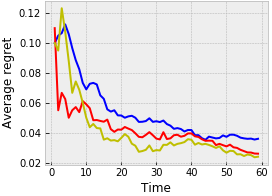}
    \caption{Omniglot (full)}
  \end{subfigure}
  \caption{Average $2$-shifted regret vs game duration $T$ for online clustering against $2$-shifted distributions. Color scheme: \textcolor{blue}{\bf Exponential Forecaster}, \textcolor{red}{\bf Fixed Share EF}, \textcolor{colorY}{\bf Generalized Share EF}}
\label{fig:1}
\end{figure*}

We observe that our algorithms have significantly lower regrets (about 15-40\% relative for the datasets considered, for $T\ge 40$) compared to the Exponential Forecaster algorithm across all datasets. We also note that the exact advantage of adding exploration to exponential updates varies with datasets and problem instances. In Appendix \ref{appendix:exp} we have compiled further experiments that reaffirm the strengths of our approach against different {\it changing environments} and also compare against the static setting.

\begin{remark}
The applications considered, for which the algorithms have been implemented and empirically evaluated, have piecewise constant utility functions with $d=1$. For these it is possible to simply maintain the weight on each piece of $\Sigma_tu_t(\rho)$ in $O(K\log Kt)$ time for round $t$ where each $u_t(\cdot)$ has $O(K)$ pieces by using a simple interval tree data structure \cite{cohen2017online}. The tree lazily maintains weight for each of $O(Kt)$ pieces, takes $O(\log Kt)$ time for lazy insertion of $O(K)$ new pieces and allows drawing with probability proportional to weight in $O(\log Kt)$ time. Similarly $O(K\log Kt)$ updates are possible for Algorithm \ref{alg:gsef} as well in this case. Section \ref{sec:sam} of the paper addresses the harder problem of polynomial time implementation for arbitrary $d$ (for Algorithm \ref{alg:fsef}).
\end{remark}

\section{Discussion and open problems}
We presented approaches which trade off exploitation with exploration for the online optimization problem to obtain low shifting regret for the case of general non-convex functions with sharp but dispersed discontinuities. Optimizing for the stronger baseline of shifting regret leads to empirically better payout, as we have shown via experiments bearing applications to algorithm configuration. Our focus here is on the full-information setting which corresponds to the entire utility function being revealed at each iteration, and we present almost tight theoretical results for the same. Other relevant settings include bandit and semi-bandit feedback where the function value is revealed for only the selected point or a subset of the space containing the point. It would be interesting to obtain low shifting regret in these settings \cite{auer2002nonstochastic}.

\section{Acknowledgements}

We thank Ellen Vitercik for helpful feedback. This work was supported in part by NSF grants CCF-1535967, IIS-1618714, IIS-1901403, CCF-1910321, SES-1919453, an Amazon Research Award, a Microsoft Research Faculty Fellowship, a Bloomberg Data Science research grant, and by the generosity of Eric and Wendy Schmidt by recommendation of the Schmidt Futures program. Views expressed in this work do not necessarily reflect those of any funding agency.

\newpage

\section*{Appendix}

\appendix
\section{Discretization based algorithm}\label{appendix:disc}

Recall that $\C\subset \R^d$ is contained in a ball of radius $R$. A standard greedy construction gives an $r$-discretization of size at most $(3R/r)^d$ \cite{balcan2018dispersion}. Given the dispersion parameter $\beta$, a natural choice is to use a $T^{-\beta}$-discretization as in Algorithm \ref{alg:dsef}. 

\disc*
\begin{proof}[Proof of Theorem \ref{lem:disc}]
We show we can round the optimal points in $\C$ to points in the $(T^{-\beta})$-discretization $\mathcal{D}$ with a payoff loss at most $(sH+L)T^{1-\beta}$ in expectation. But in $\mathcal{D}$ we know a way to bound regret by $R^{finite}(T,s,N)$, where $N$, the number of points in $\mathcal{D}$, is at most $\left(\frac{3R}{T^{-\beta}}\right)^d=\left(3RT^{\beta}\right)^d$.\\
Let $t_{0:s}$ denote the expert switching times in the optimal offline payoff, and $\rho_i^*$ be the point picked by the optimal offline algorithm in $[t_{i-1},t_i-1]$. Consider a ball of radius $T^{-\beta}$ around $\rho_i^*$. It must have some point $\hat{\rho}_i^*\in\mathcal{D}$. We then must have that $\{u_t\mid t\in[t_{i-1},t_i-1]\}$ has at most $O(T^{-\beta}T)=O(T^{1-\beta})$ discontinuities due to $\beta$-dispersion, which implies
\[\sum_{t=t_{i-1}}^{t_i - 1}u_t(\hat{\rho}_i^*) \ge \sum_{t=t_{i-1}}^{t_i - 1}u_t(\rho_i^*) - O(T^{1-\beta})H - L(t_i-t_{i-1})T^{-\beta}\]
Let $\hat{\rho}_t=\hat{\rho}_i^*$ for each $t_{i-1}\le t\le t_i -1$. Summing over $i$ gives
\[\sum_{t=1}^{T}u_t(\hat{\rho}_t) \ge OPT - O(T^{1-\beta})sH - LT^{1-\beta} = OPT - (sH+L)O(T^{1-\beta})\]
Now payoff of this algorithm is bounded above by the payoff of the optimal sequence of experts with $s$ shifts
\[\sum_{t=1}^{T}u_t(\hat{\rho}_t)\le OPT^{finite}\]
Let the finite experts algorithm with shifted regret bounded by $R^{finite}(T,s,N)$ choose $\rho_t$ at round $t$. Then, using the above inequalities,
\[\sum_{t=1}^{T}u_t(\rho_t) \ge OPT^{finite}-R^{finite}(T,s,N) \ge OPT - (sH+L)O(T^{1-\beta}) - R^{finite}(T,s,N)\]
We use this to bound the regret for the continuous case
\begin{align*}
R^{\C}(T,s) &= OPT-\sum_{t=1}^{T}u_t(\rho_t) \\&\le OPT - (OPT - (sH+L)O(T^{1-\beta}) - R^{finite}(T,s,N)) \\&=  R^{finite}(T,s,N)+(sH+L)O(T^{1-\beta})\end{align*}
\end{proof}

\section{Counterexamples}\label{appendix:ce}

We will construct problem instances where some sub-optimal algorithms mentioned in the paper suffer high regret.

We first show that the Exponential Forecaster algorithm of \cite{balcan2018dispersion} suffers linear $s$-shifted regret even for $s=2$. This happens because pure exponential updates may accumulate high weights on well-performing experts and may take a while to adjust weights when these experts suddenly start performing poorly.

\begin{lemma}\label{lem:ce}
There exists an instance where Exponential Forecaster algorithm of \cite{balcan2018dispersion} suffers linear $s$-shifted regret.
\end{lemma}
\begin{proof}
Let $\C=[0,1]$. Define utility functions
\begin{align*}
    u^{(0)}(\rho)=
    \begin{cases*}
    1&if $\rho < \frac{1}{2}$\\
    0&if $\rho \ge \frac{1}{2}$
    \end{cases*}\text{ and }u^{(1)}(\rho)=
    \begin{cases*}
    0&if $\rho < \frac{1}{2}$\\
    1&if $\rho \ge \frac{1}{2}$
    \end{cases*}
\end{align*}
Now consider the instance where $u^{(0)}(\rho)$ is presented for the first $T/2$ rounds and $u^{(1)}(\rho)$ is presented for the remaining rounds. In the second half, with probability at least $\frac{1}{2}$, the Exponential Forecaster algorithm will select a point from $[0,\frac{1}{2}]$ and accumulate a regret of $1$. Thus the expected 2-shifted regret of the algorithm is at least $\frac{T}{2}\cdot\frac{1}{2}=\Omega(T)$.
Notice that the construction does not depend on the {\it step size parameter} $\lambda$.
\end{proof}

We further look at the performance of Random Restarts EF (Algorithm \ref{alg:rref}), an easy-to-implement algorithm which looks deceptively similar to Algorithm \ref{alg:fsef}, against this adversary. Turns out Random Restarts EF may not restart frequently enough for the optimal value of the exploration parameter, and have sufficiently long chains of pure exponential updates in expectation to suffer high regret.

\begin{theorem}\label{thm:ce}
There exists an instance where Random Restarts EF (Algorithm \ref{alg:rref}) with parameters $\lambda$ and $\alpha$ as in Theorem \ref{thm:ub} suffers linear $s$-shifted regret.
\end{theorem}
\begin{proof}
The probability of pure exponential updates from $t=T/4$ through $t=3T/4$ is at least
\[(1-\alpha)^{T/2}=\left(1-\frac{1}{T-1}\right)^{T/2} > \frac{1}{2}\]
for $T>5$. By Lemma \ref{lem:ce}, this implies at least $\frac{T}{8}$ regret in this case, and so the expected regret of the algorithm is at least $\frac{T}{16}=\Omega(T)$.
\end{proof}

\section{Analysis of algorithms}\label{appendix:analysis}

In this section we will provide detailed proofs of lemmas and theorems from Section \ref{sec:ana}. We will restate them for easy reference.

\lemfsefrref*
\begin{proof}[Proof of Lemma \ref{lem:fsef-rref}] $w_{t}(\rho)=\E[\hat{w}_t(\rho)]$ implies $W_{t}=\E[\hat{W}_t]$ by Fubini's theorem (recall $\C$ is closed and bounded).
$w_{t}(\rho)=\E[\hat{w}_t(\rho)]$ follows by simple induction on $t$. In the base case, $\mathbf{z}_1$ is the empty set and $w_{1}(\rho)=1=\hat{w}_t(\rho)=\E[\hat{w}_t(\rho)]$. For $t>1$,
\begin{align*}\E[\hat{w}_t(\rho)]&=(1-\alpha)\E[e^{\lambda u_t(\rho)}\hat{w}_{t-1}(\rho)]+\frac{\alpha}{\Vol(\C)}\E\left[\int_{\C}{e^{\lambda u_t(\rho)}\hat{w}_{t-1}(\rho)d\rho}\right]&&\text{(definition of }\hat{w}_t)\\
&=(1-\alpha)e^{\lambda u_t(\rho)}\E[\hat{w}_{t-1}(\rho)]+\frac{\alpha}{\Vol(\C)}\int_{\C}{e^{\lambda u_t(\rho)}\E[\hat{w}_{t-1}(\rho)]d\rho}&&\text{(expectation is over }\mathbf{z}_t)\\
&=(1-\alpha)e^{\lambda u_t(\rho)}w_{t-1}(\rho)+\frac{\alpha}{\Vol(\C)}\int_{\C}{e^{\lambda u_t(\rho)}w_{t-1}(\rho)d\rho}&&\text{(inductive hypothesis)}\\
&=w_{t}(\rho)&&\text{(definition of }w_t)
\end{align*}
\end{proof}

\lemwtpartition*
\begin{proof}[Proof of Lemma \ref{lem:wt-partition}]
Recall that we wish to show that $\hat{w}_{T+1}(\rho)\mid s,\mathbf{t}_s$ (weights of Algorithm \ref{alg:rref} at time $T+1$ given restarts occur exactly at $\mathbf{t}_s$) can be expressed as the product of weight $\Tilde{w}(\rho;t_{s-1},t_s)$ at $\rho$ of regular Exponential Forecaster since the last restart times the normalized total weights accumulated over previous runs, i.e.
\begin{equation*}\hat{w}_{T+1}(\rho)\mid s,\mathbf{t}_s=\Tilde{w}(\rho;t_{s-1},t_s)\prod_{i=1}^{s-1}\frac{\Tilde{W}(t_{i-1},t_i)}{\Vol(\C)}\end{equation*}
We show this by induction on $s$. For $s=1$, we have no restarts and
\[\Tilde{w}(\rho;t_{s-1},t_s)\prod_{i=1}^{s-1}\frac{\Tilde{W}(t_{i-1},t_i)}{\Vol(\C)}=\Tilde{w}(\rho;t_0,t_1)\prod_{i=1}^0\frac{\Tilde{W}(t_{i-1},t_i)}{\Vol(\C)}=\Tilde{w}(\rho;1,T+1)=\hat{w}_{T+1}(\rho)\mid 1,\mathbf{t}_1\]
For $s>1$, the last restart occurs at $t_{s-1}>1$. By inductive hypothesis for time $t_{s-1}-1$ until which we've had $s-2$ restarts,
\[\hat{w}_{t_{s-1}-1}(\rho)\mid s,\mathbf{t}_s=\hat{w}_{t_{s-1}-1}(\rho)\mid s-1,\mathbf{t}_{s-1}=\Tilde{w}(\rho;t_{s-2},t_{s-1}-1)\prod_{i=1}^{s-2}\frac{\Tilde{W}(t_{i-1},t_i)}{\Vol(\C)}\]
Due to restart at $t_{s-1}$,
\[\hat{w}_{t_{s-1}}(\rho)\mid s,\mathbf{t}_s=\frac{\int_{\C}{e^{\lambda u_t(\rho)}\hat{w}_{t_{s-1}-1}(\rho)d\rho}}{\Vol(\C)}=\prod_{i=1}^{s-1}\frac{\Tilde{W}(t_{i-1},t_i)}{\Vol(\C)}\]
It's regular exponential updates from this point to $t_s$, which gives the result.

\end{proof}

\thmub*
\begin{proof}[Full proof of Theorem \ref{thm:ub}]
We first provide an upper and lower bound to $\frac{W_{T+1}}{W_{1}}$.\\ \\
{\it Upper bound}: The proof is similar to the upper bound for exponential weighted forecaster in \cite{balcan2018dispersion} and uses Lemma \ref{lem:wt-basic} for $W_t$.
\[\frac{W_{t+1}}{W_{t}} = \frac{\int_{\C}e^{\lambda u_t(\rho)}w_{t}(\rho)d\rho}{W_{t}} = \int_{\C}e^{\lambda u_t(\rho)}\frac{w_{t}(\rho)}{W_{t}}d\rho = \int_{\C}e^{\lambda u_t(\rho)}p_{t}(\rho)d\rho\]
Finally use inequalities $e^{\lambda z}\le1+(e^{\lambda}-1)z$ for $z\in[0,1]$ and $1+z\le e^z$ to get
\[\frac{W_{t+1}}{W_{t}} \le\int_{\C}p_{t}(\rho)\bigg(1+(e^{H\lambda}-1)\frac{u_t(\rho)}{H}\bigg)d\rho = 1+(e^{H\lambda}-1)\frac{P_t}{H} \le \exp\bigg((e^{H\lambda}-1)\frac{P_t}{H}\bigg)\]
where $P_t$ denotes the expected payoff of the algorithm in round $t$. Let $P(\A)$ be the expected total payoff. Then we can write $\frac{W_{T+1}}{W_{1}}$ as a telescoping product which gives
\begin{equation}\label{eq:wt-ub-app-1}\frac{W_{T+1}}{W_{1}}=\prod_{t=1}^{T}\frac{W_{t+1}}{W_{t}}\le \exp\bigg((e^{H\lambda}-1)\frac{\sum_tP_t}{H}\bigg) = \exp\bigg(\frac{P(\mathcal{A})(e^{H\lambda}-1)}{H}\bigg)\end{equation}
\\{\it Lower bound}: Again the proof is similar to \cite{balcan2018dispersion} and the major difference is use of Lemma \ref{lem:wt-partition}.
\\ We first lower bound payoffs of points close to the optimal sequence of experts using dispersion.
If the optimal sequence with $s$ shifts has shifts at $t^*_i$ ($1\le i\le s-1$), by $\beta$-dispersion for any $\rho_i\in\B(\rho_i^*,w)$
\begin{equation}
\label{eq:dispersion-app}
\sum_{t=t^*_{i-1}}^{t^*_i - 1}u_t(\rho_i) \ge \sum_{t=t^*_{i-1}}^{t^*_i - 1}u_t(\rho_i^*) - kH - L(t^*_i-t^*_{i-1})w
\end{equation}
where $w=T^{-\beta}$ and $k=O(T^{1-\beta})$. Summing both sides over $i\in[s-1]$ helps us relate the lower bound to the payoff $OPT$ of the optimal sequence.
\begin{equation}
\label{eq:summation-app}
\sum_{i=1}^s\sum_{t=t^*_{i-1}}^{t^*_i - 1}u_t(\rho_i) \ge \sum_{i=1}^s\sum_{t=t^*_{i-1}}^{t^*_i - 1}u_t(\rho_i^*) - kH - L(t^*_i-t^*_{i-1})w = OPT - ksH - LTw
\end{equation}
Now to lower bound $\frac{W_{T+1}}{W_{1}}$, we first lower bound $W_{T+1}$. We use Lemma \ref{lem:wt-partition} and lower bound by picking the term corresponding to times of expert shifts in the optimal sequence with $s$-shifted expert.
\begin{align}
W_{T+1}&=\sum_{ s\in [T]}\bigg[\sum_{t_0=1<t_1\dots<t_s=T+1}\bigg(\frac{\alpha^{s-1}(1-\alpha)^{T-s}}{\Vol(\C)^{s-1}}\prod_{i=1}^s\Tilde{W}(t_{i-1},t_i)\bigg)\bigg]&&\text{(Lemma \ref{lem:wt-partition})}\\
&\ge  \frac{\alpha^{s-1}(1-\alpha)^{T-s}}{\Vol(\C)^{s-1}}\prod_{i=1}^s\Tilde{W}(t^*_{i-1},t^*_i)
\label{eq:wt-lb-app}
\end{align}
The product of $\Tilde{W}$'s can in turn be lower bounded by restricting attention to points close (i.e. within a ball of radius $w$ centered at optimal expert $\rho_i^*$) to the optimal sequence. The payoffs of such points was lower-bounded in (\ref{eq:dispersion-app}) and (\ref{eq:summation-app}) in terms of the optimal payoff.
\begin{align*}
\prod_{i=1}^s&\Tilde{W}(t^*_{i-1},t^*_i)=\prod_{i=1}^s\int\displaylimits_{\C}\exp\bigg({\lambda \sum_{t=t^*_{i-1}}^{t^*_i-1}u_t(\rho)}\bigg)d\rho\\&\ge \prod_{i=1}^s\int\displaylimits_{\B(\rho_i^*,w)}\exp\bigg({\lambda \sum_{t=t^*_{i-1}}^{t^*_i-1}u_t(\rho)}\bigg)d\rho&&\text{(Restrict integration domains)}\\
&\ge \prod_{i=1}^s\int\displaylimits_{\B(\rho_i^*,w)}\exp\bigg({\lambda\big( \sum_{t=t^*_{i-1}}^{t^*_i - 1}u_t(\rho_i^*) - kH - L(t^*_i-t^*_{i-1})w}\big)\bigg)d\rho&&(\text{Using equation \ref{eq:dispersion-app})}\\
&= \Vol(\B(w))^s\exp\Bigg(\sum_{i=1}^s{\lambda\bigg( \sum_{t=t^*_{i-1}}^{t^*_i - 1}u_t(\rho_i^*) - kH - L(t^*_i-t^*_{i-1})w}\bigg)\Bigg)&&\text{(Integrand independent of $\rho$)}\\
&= \Vol(\B(w))^s\exp\bigg(\lambda\big(OPT - ksH - LTw\big)\bigg)&&\text{(Using equation \ref{eq:summation-app})}
\end{align*}
Plugging into equation (\ref{eq:wt-lb-app}) we get
\begin{align*}
W_{T+1}\ge  \frac{\alpha^{s-1}(1-\alpha)^{T-s}\Vol(\B(w))^s}{\Vol(\C)^{s-1}}\exp\bigg(\lambda\big(OPT - ksH - LTw\big)\bigg)
\end{align*}
Also, $W_{1}=\int_{\C}w_1(\rho)d\rho=\Vol(\C)$. Thus, using the fact that ratio of volume of balls $\B(w)$ and $\B(R)$ in $d$-dimensions is $(w/R)^d$, and assuming $\C$ is bounded by some ball $\B(R)$.
\[\frac{W_{T+1}}{W_1}\ge\frac{\alpha^{s-1}(1-\alpha)^{T-s}\Vol(\B(w))^s}{\Vol(\C)^{s}}\exp\bigg(\lambda\big(OPT - ksH - LTw\big)\bigg)\]
\begin{equation}\label{eq:wt-lb-app-1}\:\:\:\ge\alpha^{s-1}(1-\alpha)^{T-s}\bigg(\frac{w}{R}\bigg)^{sd}\exp\bigg(\lambda\big(OPT - ksH - LTw\big)\bigg)\end{equation}
{\it Putting together}: Combining upper and lower bounds from (\ref{eq:wt-ub-app}) and (\ref{eq:wt-lb-app-1}) respectively,
\[\log\big(\alpha^{s-1}(1-\alpha)^{T-s}\big)-sd\log\frac{R}{w}+\lambda(OPT- ksH - LTw) \le \frac{P(\mathcal{A})(e^{H\lambda}-1)}{H}\]
which rearranges to
\[OPT-P(\mathcal{A}) \le P(\mathcal{A})\frac{(e^{H\lambda}-1-H\lambda)}{H\lambda}+\frac{sd\log(R/w)}{\lambda}+ksH+LTw-\frac{\log(\alpha^{s-1}(1-\alpha)^{T-s})}{\lambda}\]
Using $P(\mathcal{A})\le HT$ and using $e^z\le 1+z+(e-2)z^2$ for $z\in[0,1]$ we have
\begin{align*}
OPT-P(\mathcal{A}) &\le HT\frac{(e^{H\lambda}-1-H\lambda)}{H\lambda}+\frac{sd\log(R/w)}{\lambda}+ksH+LTw-\frac{\log(\alpha^{s-1}(1-\alpha)^{T-s})}{\lambda}\\
&< H^2T\lambda+\frac{sd\log(R/w)}{\lambda}+ksH+LTw-\frac{\log(\alpha^{s-1}(1-\alpha)^{T-s})}{\lambda}
\end{align*}
Now we tighten the bound, first w.r.t. $\alpha$ then w.r.t. $\lambda$. Note $\min_{\alpha}-\log(\alpha^{s-1}(1-\alpha)^{T-s})$ occurs for $\alpha_0=\frac{s-1}{T-1}$ and
\[-\log(\alpha_0^{s-1}(1-\alpha_0)^{T-s})=(T-1)\bigg[-\frac{s-1}{T-1}\log\frac{s-1}{T-1}-\frac{T-s}{T-1}\log\frac{T-s}{T-1}\bigg]\le (s-1)\log e\frac{T-1}{s-1}\]
(binary entropy function satisfies $h(x)\le x\ln (e/x)$ for $x\in[0,1]$). Finally minimizing over $\lambda$ gives
\[OPT-P(\mathcal{A}) \le O(H\sqrt{sT(d\log (R/w)+\log(T/s))} + ksH + LTw)\]
for $\lambda=\sqrt{s(d\log(R/w)+\log(T/s))/T}/H$.
Plugging back $w=T^{-\beta}$ and $k=O(T^{1-\beta})$ completes the proof.
\end{proof}
The rest of this section is concerned with the analysis of Algorithm \ref{alg:gsef} for the sparse experts setting.
\begin{lemma}
\label{lem:wt-partition-gsef2}
For any $t<T$,
$$w_{T}(\rho)\ge \alpha(1-\alpha)^{T-t}\pi_{t}(\rho)\Tilde{w}(\rho;t,T)W_t$$
\end{lemma}
\begin{proof}
Follows using the restart algorithm technique used in Lemmas \ref{lem:wt-partition} and \ref{lem:wt-partition-gsef}. Consider the probability of last {\it restart} being at time $t$.
Notice this also implies Corollary \ref{cor:wt-partition-gsef}.
\end{proof}
\begin{lemma}
\label{lem:sparse}
Let $\pi_t(\rho)=\sum_{i=1}^t\beta_{i,t}p_i(\rho)$ in Algorithm \ref{alg:gsef}. Then
\[\pi_t(\rho)=\frac{\alpha_{1,t}}{W_1}+\sum_{i=1}^{t-1}\alpha_{i+1,t}\frac{e^{\lambda u_{i}(\rho)}w_{i}(\rho)}{W_{i+1}}\]
where
\[\alpha_{i,t}\ge \frac{1-\alpha}{e_t}\left(e^{-\gamma}+\frac{\alpha}{e_t}\right)^{t-i}\]
and $e_t:=\sum_{i=1}^te^{-\gamma(i-1)}$.
\end{lemma}
\begin{proof}
Notice, by definition of weight update in Algorithm \ref{alg:gsef},
$$p_{t}(\rho)=(1-\alpha)\frac{e^{\lambda u_{t-1}(\rho)}w_{t-1}(\rho)}{W_t}+\alpha\sum_{i=1}^{t-1}\beta_{i,t-1}p_i(\rho)=(1-\alpha)\frac{e^{\lambda u_{t-1}(\rho)}w_{t-1}(\rho)}{W_t}+\alpha\pi_{t-1}(\rho)$$
This gives us a recursive relation for $\alpha_{i,t}$.
\begin{align*}
    \alpha_{i,t}=\begin{cases*}
    \beta_{i,t}(1-\alpha)+\alpha\sum_{j=i+1}^t\beta_{j,t}\alpha_{i,j-1}&if $i>1$\\
    \beta_{i,t}+\alpha\sum_{j=i+1}^t\beta_{j,t}\alpha_{i,j-1}&if $i=1$
    \end{cases*}
\end{align*}
Thus for each $1\le i\le t$
\[ \alpha_{i,t}\ge\beta_{i,t}(1-\alpha)+\alpha\sum_{j=i+1}^t\beta_{j,t}\alpha_{i,j-1}\]
We proceed by induction on $t-i$. For $i=t$,
\[\alpha_{t,t}\ge\beta_{t,t}(1-\alpha)=\frac{1-\alpha}{e_t}\left(e^{-\gamma}+\frac{\alpha}{e_t}\right)^{t-t}\]
For $i<t$, by inductive hypothesis
\begin{align*} 
\alpha_{i,t}&\ge\beta_{i,t}(1-\alpha)+\sum_{j=i+1}^t\beta_{j,t}\alpha\alpha_{i,j-1}\\
&\ge (1-\alpha)\frac{e^{-\gamma(t-i)}}{e_t}+\alpha\frac{1-\alpha}{e_t}\sum_{j=i+1}^t\frac{e^{-\gamma(t-j)}}{e_t}\left(e^{-\gamma}+\frac{\alpha}{e_t}\right)^{j-1-i}\\
&=\frac{1-\alpha}{e_t}\left(e^{-\gamma(t-i)}+\frac{\alpha e^{-\gamma t}}{e_t}\sum_{j=i+1}^te^{\gamma j}\left(e^{-\gamma}+\frac{\alpha}{e_t}\right)^{j-1-i}\right)\\
&=\frac{1-\alpha}{e_t}\left(e^{-\gamma(t-i)}+\frac{\alpha e^{-\gamma t}}{e_t}\frac{e^{\gamma(t+1)}\left(e^{-\gamma}+\frac{\alpha}{e_t}\right)^{t-i}-e^{\gamma(i+1)}}{e^{\gamma}\left(e^{-\gamma}+\frac{\alpha}{e_t}\right)-1}\right)\\
&=\frac{1-\alpha}{e_t}\left(e^{-\gamma}+\frac{\alpha}{e_t}\right)^{t-i}
\end{align*}
which completes the induction step.
\end{proof}

\begin{corollary}
\label{cor:sparse}
Let $w_t(\rho), W_t$ be as in Algorithm \ref{alg:gsef} and $\pi_t$ as in Lemma \ref{lem:wt-partition-gsef}. For each $\tau<\tau'<t$ and any bounded $f$ defined on $\C$.
\[\int_{\C}\pi_{t}(
\rho)f(\rho)d\rho\ge \frac{\alpha(1-\alpha)^{\tau'-\tau}(1-e^{-\gamma})}{\left(e^{-\gamma}+\alpha(1-e^{-\gamma})\right)^{\tau'-t}}\frac{W_{\tau}}{W_{\tau'}}\int_{\C}\pi_{\tau}(\rho)\Tilde{w}(\rho;\tau,\tau')f(\rho)d\rho\]
\end{corollary}
\begin{proof}
By Lemma \ref{lem:sparse},
\begin{align*}
    \int_{\C}\pi_{t}(
\rho)f(\rho)d\rho&=\int_{\C}\pi_{t}(
\rho)f(\rho)d\rho\\
&\ge \int_{\C}\alpha_{\tau',t}\frac{e^{\lambda u_{\tau'-1}(\rho)}w_{\tau'-1}(\rho)}{W_{\tau'}}f(\rho)d\rho\\
&\ge\frac{1-\alpha}{e_t}\left(e^{-\gamma}+\frac{\alpha}{e_t}\right)^{t-\tau'}\frac{1}{W_{\tau'}}\int_{\C}e^{\lambda u_{\tau'-1}(\rho)}w_{\tau'-1}(\rho)f(\rho)d\rho\\
&\ge\frac{1-\alpha}{e_t}\left(e^{-\gamma}+\frac{\alpha}{e_t}\right)^{t-\tau'}\frac{\alpha(1-\alpha)^{\tau'-1-\tau}W_{\tau}}{W_{\tau'}}\int_{\C}\pi_{\tau}(\rho)\Tilde{w}(\rho;\tau,\tau')f(\rho)d\rho
\end{align*}
where for the last inequality we have used Lemma \ref{lem:wt-partition-gsef2}. The lemma then follows by noting
\[\frac{1}{e_t}=\frac{1-e^{-\gamma}}{1-e^{-\gamma t}}\ge 1-e^{-\gamma}\]
where $e_t=\sum_{i=1}^te^{-\gamma(i-1)}$ as defined in Lemma \ref{lem:sparse}.
\end{proof}

\thmubsparse*
\begin{proof}[Proof of Theorem \ref{thm:ub-sparse}]
Like Theorem \ref{thm:ub} we first provide an upper and lower bound to $\frac{W_{T+1}}{W_{1}}$. The upper bound proof is identical to that of Theorem \ref{thm:ub} by replacing Lemma \ref{lem:wt-basic} by Lemma \ref{lem:wt-basic-gsef}.

For the lower bound we use Corollaries \ref{cor:wt-partition-gsef} and \ref{cor:sparse}. Applying corollary \ref{cor:sparse} repeatedly to collect exponential updates for the times OPT played the same expert lets us use the arguments for Theorem \ref{thm:ub} to get Equation \ref{eq:wt-lb-gsef}. Indeed if $\{(s_i,f_i)\mid 1\le i \le l\}$ are the start and finish times of a particular expert $\rho$ in the OPT sequence, we can use Corollary \ref{cor:wt-partition-gsef} to write
\begin{align*}
    W_{f_l+1}&\ge \alpha(1-\alpha)^{f_l+1-s_l}W_{s_l}\Tilde{W}(\pi_{s_l};s_l,f_l+1)
\end{align*}
Applying Corollary \ref{cor:sparse} repeatedly now gets us
\begin{align*}
    W_{f_l+1}&\ge \frac{\alpha^l(1-\alpha)^{\sum_{j=1}^{l}f_{j}+1-s_{j}}(1-e^{-\gamma})^{l-1}}{(e^{-\gamma}+\alpha(1-e^{-\gamma}))^{\sum_{j=1}^{l-1}f_{j}+1-s_{j+1}}}\frac{\prod_{i=1}^lW_{s_i}}{\prod_{i=1}^{l-1}W_{f_i+1}}\int_{\C}\left(\pi_{s_1}(\rho)\prod_{j=1}^{l}\Tilde{w}(\rho;s_j,f_{j}+1)\right)d\rho
\end{align*}
or
\begin{align*}
   \frac{\prod_{i=1}^{l}W_{f_i+1}}{\prod_{i=1}^lW_{s_i}}&\ge \frac{\alpha^l(1-\alpha)^{\sum_{j=1}^{l}f_{j}+1-s_{j}}(1-e^{-\gamma})^{l-1}}{(e^{-\gamma}+\alpha(1-e^{-\gamma}))^{\sum_{j=1}^{l-1}f_{j}+1-s_{j+1}}}\int_{\C}\left(\pi_{s_1}(\rho)\prod_{j=1}^{l}\Tilde{w}(\rho;s_j,f_{j}+1)\right)d\rho
\end{align*}
Multiplying these inequalities for each of $m$ experts in the optimal sequence gives us $\frac{W_{T+1}}{W_{1}}$ on the left side. Also note
\[\int_{\C}\pi_{t}(\rho)f(\rho)d\rho \ge \frac{\alpha_{1,t}}{W_1}\int_{\C}f(\rho)d\rho\]
and, using dispersion as in proof of Theorem \ref{thm:ub},
\[\prod_{\text{experts in OPT}}\int_{\C}\left(\prod_{j=1}^{l}\Tilde{w}(\rho;s_j,f_{j}+1)\right)d\rho\ge \Vol(\B(T^{-\beta}))^m\exp\left(\lambda\left(OPT - (mH + L)O(T^{1-\beta})\right)\right)\]
Putting it all together gives Equation \ref{eq:wt-lb-gsef}. Combining the lower and upper bounds on $\frac{W_{T+1}}{W_{1}}$ gives us a bound on $OPT-P(\A)$. 
\begin{align*}
OP&T-P(\mathcal{A}) < H^2T\lambda+\frac{md\log(RT^{\beta})}{\lambda}+(mH+L)O(T^{1-\beta})-\log\left(\frac{\alpha^{s}(1-\alpha)^{T}(1-e^{-\gamma})^{s}}{(e^{-\gamma}+\alpha(1-e^{-\gamma}))^{-mT}}\right)/\lambda
\end{align*}
We now chose parameters $\gamma,\alpha,\lambda$ to get the tightest regret bound. Note that $-\log(\alpha^{s}(1-\alpha)^{T})$ is minimized for $\alpha=\frac{s}{T+s}=\Theta(\frac{s}{T})$ and $-\log((1-e^{-\gamma})^{s}(e^{-\gamma}+\alpha(1-e^{-\gamma}))^{mT})$ is minimized for $\gamma=\log\left(\frac{1+s/mT}{1-s\alpha/mT(1-\alpha)}\right)=\Theta(\frac{s}{mT})$. The corresponding minimum values can be bounded as
\[-\log(\alpha^{s}(1-\alpha)^{T})=s\log\frac{T+s}{s}+T\log\left(1+\frac{s}{T}\right)\le s\log\frac{T+s}{s}+s=O\left(s\log\frac{T}{s}\right)\]
using $\log(1+x)\le x$, and substituting $e^{-\gamma}=\frac{1-s\alpha/mT(1-\alpha)}{1+s/mT}$
\begin{align*}-\log((1-e^{-\gamma})^{s}(e^{-\gamma}+\alpha(1-e^{-\gamma}))^{mT})&=-s\log\frac{\frac{s}{mT}\cdot\frac{1}{(1-\alpha)}}{1+\frac{s}{mT}}-mT\log\frac{1}{1+\frac{s}{mT}}\\
&=s\log\left((1-\alpha)\left(\frac{mT}{s}+1\right)\right)+mT\log\left(1+\frac{s}{mT}\right)\\
&\le s\log\left((1-\alpha)\left(\frac{mT}{s}+1\right)\right) + 1
\\&=O\left(s\log\frac{mT}{s}\right)
\end{align*}
Finally we minimize w.r.t. $\lambda$, to obtain the desired regret bound.
\end{proof}

\section{Adaptive Regret}\label{appendix:adaptive}

It is known that the fixed share algorithm obtains good adaptive regret for finite experts and OCO \cite{adamskiy2012closer}. We show that it is the case here as well.

\begin{defn}
\label{prob:adaptive}
The $\tau$-adaptive regret (due to \cite{hazan2007adaptive}) is given by
\[\E\Bigg[\max_{\substack{\rho^*\in\C,\\
1\le r<s\le T, s-r\le \tau}}\sum_{t=r}^{s}(u_t(\rho^*)-u_t(\rho_t))\Bigg]\]
\end{defn}
The goal here is to ensure small regret on all intervals of size up to $\tau$ simultaneously.  Adaptive regret
measures how well the algorithm approximates the best expert locally,
and it is therefore somewhere between the static regret (measured on
all outcomes) and the shifted regret, where the algorithm is compared
to a good sequence of experts.

\begin{theorem}
\label{thm:ub-adaptive}
Algorithm \ref{alg:fsef} enjoys  $O(H\sqrt{\tau (d\log( R/w)+\log\tau)}+(H+L)\tau^{1-\beta})$ $\tau$-adaptive regret for  $\lambda=\sqrt{(d\log(R\tau^{\beta})+\log(\tau))/\tau}/H$ and $\alpha=1/\tau$.
\end{theorem}
\begin{proof}[Proof sketch of Theorem \ref{thm:ub-adaptive}]
Apply arguments of Theorem \ref{thm:ub} to upper and lower bound $W_{s+1}/W_r$ for any interval $[r,s]\subseteq [1,T]$ of size $\tau$. 
We get
\begin{equation*}\label{eq:wt-ub-app}\frac{W_{s+1}}{W_{r}}\le  \exp\bigg(\frac{P(\mathcal{A})(e^{H\lambda}-1)}{H}\bigg)\end{equation*}
where $P(\A)$ is the expected payoff of the algorithm in $[r,s]$,
Also, by Corollary \ref{cor:wt-partition-gsef} (equivalent for Algorithm \ref{alg:fsef})
\[W_{s+1}\ge \frac{\alpha(1-\alpha)^{s+1-r}}{\Vol(\C)}\Tilde{W}(r,s)W_r=\frac{\alpha(1-\alpha)^{\tau}}{\Vol(\C)}\Tilde{W}(r,s)W_r\]
By dispersion, as in proof of Theorem \ref{thm:ub},
\[\Tilde{W}(r,s)\ge \Vol(\B(\tau^{-\beta}))\exp\left(\lambda\left(OPT - (H + L)O(\tau^{1-\beta})\right)\right)\]
Putting the upper and lower bounds together gives us a bound on $OPT-P(\A)$, which gives the desired regret bound for $\alpha=\frac{1}{\tau}$.
\end{proof}

\section{Efficient Sampling}\label{appendix:efficient}

In Section \ref{sec:sam} we introduced Algorithm \ref{alg:fsef-d} for efficient implementation of Algorithm \ref{alg:fsef} in $\R^d$. We present proofs of the results in that section, and an exact algorithm for the case $d=1$.

\begin{algorithm}
\caption{Fixed Share Exponential Forecaster - exact algorithm for one dimension}
\label{alg:fsef-1d}
{\bf Input:} $\lambda \in (0, 1/H]$
\begin{enumerate}[wide, labelwidth=!, labelindent=0pt]
    \item[1.] $W_1=\Vol(\C)$
    \item[2.] For each $t=1,2,\dots,T$:
    \begin{enumerate}
        \item[] Estimate $C_{t,j}$ using Lemma \ref{lem:pt-recursion} for each $1\le j \le t$ using memoized values for weights.
        \item[] Sample $i$ with probability $C_{t,i}$.
        \item[] Sample $\rho$ with probability proportional to $\Tilde{w}(\rho;i,t)$.
        \item[] Estimate $W_{t+1}$ using Lemma \ref{lem:wt-recursion}.
    \end{enumerate}
\end{enumerate}
\end{algorithm}

\lemwtrecursion*
\begin{proof}[Proof of Lemma \ref{lem:wt-recursion}]
For $t=1$, first term is $\Tilde{W}(1,2)=\int_{\C}e^{\lambda u_{1}(\rho)}d\rho=W_2$ and second term is zero. Also, by Lemma \ref{lem:wt-basic}, for $t>1$
\[W_{t+1} = \int_{\C}e^{\lambda u_{t}(\rho)}w_{t}(\rho)d\rho = \int_{\C}e^{\lambda u_{t}(\rho)}\bigg[(1-\alpha)e^{\lambda u_{t-1}(\rho)}w_{t-1}(\rho)+\frac{\alpha}{\Vol(\C)}\int_{\C}{e^{\lambda u_{t-1}(\rho)}w_{t-1}(\rho)d\rho}\bigg]d\rho\]
\[=(1-\alpha)\int_{\C}e^{\lambda (u_{t}(\rho)+u_{t-1}(\rho))}w_{t-1}(\rho)d\rho + \frac{\alpha}{\Vol(\C)}W_{t}\int_{\C}e^{\lambda u_{t}(\rho)}d\rho\]
Continue substituting $w_{j}(\rho)=(1-\alpha)e^{\lambda u_j(\rho)}w_{j-1}(\rho)+\frac{\alpha}{\Vol(\C)}\int_{\C}{e^{\lambda u_j(\rho)}w_{j-1}(\rho)d\rho}$ in the first summand until $w_1=1$ to get the desired expression.
\end{proof}
\begin{defn}
For $\alpha\ge 0$ we say $\hat{A}$ is an $(\alpha,\zeta)$-approximation of $A$ if
\[Pr\big(e^{-\alpha}A\le \hat{A} \le e^{\alpha}A \big)\ge 1-\zeta\]
\end{defn}

\begin{lemma}
\label{lem:approx}
If $\hat{A}$ is an $(\alpha,\zeta)$-approximation of $A$ and $\hat{B}$ is a $(\beta,\zeta')$-approximation of $B$, such that $A,B,\hat{A},\hat{B}$ are all positive reals
\begin{enumerate}
    \item $\hat{A}\hat{B}$ is an $(\alpha+\beta,\zeta+\zeta')$-approximation of $AB$
    \item $p\hat{A}+q\hat{B}$ is a $(\max\{\alpha,\beta\},\zeta+\zeta')$-approximation of $pA+qB$ for $p,q\ge 0$
\end{enumerate}
\end{lemma}
\begin{proof}
The results follow from union bound on failure probabilities.
\end{proof}
\begin{corollary}
\label{cor:wt}
For one-dimensional case, we can exactly compute $\Tilde{W}(i,j),\:\:1\le i <j\le t$, hence $W_t$ at each iteration can be computed in $O(t)$ time using Lemma \ref{lem:wt-recursion}. More generally, if we have a $(\beta,\zeta)$ approximation for each $\Tilde{W}(i,j),\:\:1\le i <j\le t$, then by Lemma \ref{lem:wt-recursion} we can compute a $(t\beta,t^2\zeta)$-approximation for $W_{t+1}$. 
\end{corollary}
\begin{proof}
Union bound on failure probabilities of all $\Tilde{W}(i,j),\:\:1\le i <j\le t$ gives we have a $\beta$ approximation for each with probability at least $1-t^2\zeta$. This covers failure for all terms in $W_i,2\le i\le t$. Further, by induction, the error for estimates for $W_i$ is at most $(i-1)\beta$. By Lemma \ref{lem:approx}, the error for $W_{t+1}$ estimates is at most $t\beta$.
\end{proof}

\lemptrecursion*
\begin{proof}[Proof of Lemma \ref{lem:pt-recursion}] At each iteration, $p_t$ is obtained by mixing $e^{u_t}p_{t-1}$ with the uniform distribution, i.e. we rescale distributions that $p_{t-1}$ was a mixture of and add one more. Another way to view it is to consider a distribution over the sequences of exponentially updated or randomly chosen points. The final probability distribution is the mixture of a combinatorial number of distributions but a large number of them have a proportional density. $C_{t,i}$ are simply sums of mixture coefficients. This establishes the intuition for the expression for $p_t$ and that the mixing coefficients should sum to 1, but we still need to convince ourselves that the coefficients can be computed efficiently.\\
We proceed by induction on $t$. For $t=1$ (using definitions for $w_2(\rho)$ and $w_2(\rho)$)
\[p_1(\rho)=\frac{w_1(\rho)}{W_1}=\frac{1}{\Vol(\C)}=C_{1,1}\frac{\Tilde{w}(\rho;1,1)}{\Tilde{W}(1,1)}\]
(recall $\Tilde{w}(\rho;1,1):=1$ and $\Tilde{W}(1,1)=\int_{\C}\Tilde{w}(\rho;1,1)d\rho$). For the inductive step, we first express $p_{t+1}$ in terms of $p_t$
\begin{align*}
p_{t+1}(\rho)&=\frac{w_{t+1}(\rho)}{W_{t+1}}\\&=(1-\alpha)\frac{e^{\lambda u_t(\rho)}w_{t}(\rho)}{W_{t+1}}+\frac{\alpha}{\Vol(\C)}\\&=(1-\alpha)\frac{W_{t}}{W_{t+1}}\frac{e^{\lambda u_t(\rho)}w_{t}(\rho)}{W_{t}}+\frac{\alpha}{\Vol(\C)}\\
&=(1-\alpha)\frac{W_{t}}{W_{t+1}}e^{\lambda u_t(\rho)}p_t(\rho)+\frac{\alpha}{\Vol(\C)}
\end{align*}
The lemma is now straightforward to see with induction hypothesis.
\begin{align*}
p_{t+1}(\rho)&=(1-\alpha)\frac{W_{t}}{W_{t+1}}e^{\lambda u_t(\rho)}\bigg[\sum_{i=1}^{t}C_{t,i}\frac{\Tilde{w}(\rho;i,t)}{\Tilde{W}(i,t)}\bigg]+\frac{\alpha}{\Vol(\C)}\\
&=\sum_{i=1}^{t}\bigg[(1-\alpha)\frac{W_{t}}{W_{t+1}}C_{t,i}\frac{\Tilde{w}(\rho;i,t+1)}{\Tilde{W}(i,t)}\bigg] +\frac{\alpha}{\Vol(\C)}\\
&=\sum_{i=1}^{t}\bigg[\bigg((1-\alpha)\frac{W_{t}}{W_{t+1}}\frac{\Tilde{W}(i,t+1)}{\Tilde{W}(i,t)}C_{t,i}\bigg)\frac{\Tilde{w}(\rho;i,t+1)}{\Tilde{W}(i,t+1)}\bigg] +\frac{C_{t+1,t+1}}{\Vol(\C)}\\
&=\sum_{i=1}^{t}C_{t+1,i}\frac{\Tilde{w}(\rho;i,t+1)}{\Tilde{W}(i,t+1)}+\frac{C_{t+1,t+1}}{\Vol(\C)}
\end{align*}
Finally noting
\[C_{t+1,t+1}\frac{\Tilde{w}(\rho;t+1,t+1)}{\Tilde{W}(t+1,t+1)}=C_{t+1,t+1}\frac{1}{\int_{\C}(1)d\rho} = \frac{C_{t+1,t+1}}{\Vol(\C)}\]
completes the proof.
\\ \\
Thus $W_{t}$ (by Lemma \ref{lem:wt-recursion}) and $C_{t,i}$ can be computed recursively for logconcave utility functions using integration algorithm from \cite{lovasz2006fast}. We can compute them efficiently using Dynamic Programming.\\
Finally it's straightforward to establish that the coefficients for $p_t$ must lie on the probability simplex $\Delta^{t-1}$. All coefficients are positive, which is easily seen from the recursive relation and noting all weights are positive. Also we know
\[p_t(\rho)=\sum_{i=1}^{t}C_{t,i}\frac{\Tilde{w}(\rho;i,t)}{\Tilde{W}(i,t)}\]
Since $p_t(\rho)$ is a probability distribution by definition, integrating both sides over $\C$ gives
\begin{align*}
    \int_{\C}p_t(\rho)d\rho & =\sum_{i=1}^{t}C_{t,i}\frac{\int_{\C}\Tilde{w}(\rho;i,t)d\rho}{\Tilde{W}(i,t)}&&or,\\
   1 &=\sum_{i=1}^{t}C_{t,i}
\end{align*}

\end{proof}
\begin{corollary}
If we have a $(\beta,\zeta)$ approximation for each $\Tilde{W}(i,j),\:\:1\le i <j\le t$, then by Corollary \ref{cor:wt} and Lemma \ref{lem:pt-recursion} we can compute $\hat{C}_{t+1,i}$ which are $(2t\beta,t^2\zeta)$-approximation for each $C_{t+1,i}$.
\label{cor:ct}
\end{corollary}
\begin{proof}
For $i=t$, we know $C_{t,i}$ exactly by Lemma \ref{lem:pt-recursion}. For $i<t$,
\begin{equation}
    \label{eq:cti}
    C_{t,i}=(1-\alpha)^{t-i}\frac{W_i}{W_t}\frac{\Tilde{W}(i,t)}{\Vol(\C)}C_{i,i}
\end{equation}
In Corollary \ref{cor:wt}, we show how to compute $((i-1)\beta,(i-1)^2\zeta)$-approximation for $W_i$ and $((t-1)\beta,(t-1)^2\zeta)$-approximation for $W_t$ given $(\beta,\zeta)$ approximations for each $\Tilde{W}(i,j),\:\:1\le i <j\le t$. A similar argument using Lemma \ref{lem:approx} shows with failure probability at most $t^2\zeta$, plugging in the approximations in equation \ref{eq:cti} has at most $(t+i)\beta$ error.
\end{proof}

\thmsampling*
\begin{proof}[Proof of Theorem \ref{thm:sampling}]
Based on Lemma \ref{lem:pt-recursion}, we can sample a uniformly random number $r$ in $[0,1]$ and then sample a $\rho$ from one of $t$ distributions (selected based on $r$) that $p_t(\rho)$ is a mixture of with probability proportional to $C_{t,i}$. The sampling from the exponentials can be done in polynomial time for concave utility functions using sampling algorithm of \cite{bassily2014private}. At each round we sample from exactly one of $t$ distributions in the sum for $p_t$ in Lemma \ref{lem:pt-recursion}. We compute $(\eta/6T,\zeta/2T^2)$ approximations for $\Tilde{W}(i,j),\:\:1\le i <j\le T$ in time $O(T^2K.T_{\int})$ where $T_{\int}$ is the time to integrate a logconcave distribution (at most $\Tilde{O}(d^4/\epsilon^2)$ from \cite{lovasz2006fast}). These give $(\eta/3,\zeta/2)$-approximation for $C_{t,i}$'s by corollary \ref{cor:ct}. Finally we run Algorithm 2 from \cite{balcan2018dispersion} with approximation-confidence parameters $(\eta/3,\zeta/2)$. \\
With probability at least $1-\zeta$, $C_{t,i}$ estimation and $\rho$ sampling according to $\Tilde{w}(\rho;i,t)$ succeeds. If $\hat{\mu}$ denotes output distribution of $\rho$ with approximate sampling, and $\mu$ denotes the exact distribution per $p_t(\rho)$, then we show $D_{\infty}(\hat{\mu},\mu)\le \eta$. Indeed, for any set of outcomes $E\subset \C$
\[\hat{\mu}(E) = \text{Pr}(\hat{\rho} \in E)
= \sum_{i=1}^t
\text{Pr}(\hat{\rho} \in E\mid E_{i,t}) \text{Pr}(E_{i,t}) = \sum_{i=1}^t
\hat{\mu}_i(E) \frac{\hat{C}_{t,i}}{\sum_j \hat{C}_{t,j}}\]
where $E_{i,t}$ denotes the event that $\Tilde{w}(\rho;i,t)$ was used for sampling $p_t(\rho)$, and $\hat{\mu}_i$ corresponds to the distribution for approximate sampling of $\Tilde{w}(\rho;i,t)$. Noting that we used $\eta/3$ approximation for $\hat{\mu}_i$ and each $\hat{C}_{t,i}$, we have
\[\hat{\mu}(E) \le \sum_{i=1}^t
e^{\eta/3}\mu_i(E) e^{2\eta/3}\frac{C_{t,i}}{\sum_j C_{t,j}}= e^{\eta}\mu(E)\]
Similarly, $\hat{\mu}(E)\ge e^{-\eta}\mu(E)$ and hence $D_{\infty}(\hat{\mu},\mu)\le \eta$.\\
Finally we can show (cf. Theorem 12 of \cite{balcan2018dispersion}) that with probability at least $1-\zeta$ the expected utility per round of the approximate sampler is at most a
$(1 - \eta)$ factor smaller than the expected utility per round of the exact sampler. Together with failure probability of $\zeta$, this implies at most $(\eta+\zeta)HT$ additional regret which results in same asymptotic regret as the exact algorithm for $\eta=\zeta=1/\sqrt{T}$.\\

To compute the time complexity, we note from \cite{lovasz2006fast} that logconcave functions can be integrated in $\Tilde{O}(d^4/\epsilon^2)$ and sampled from in $\Tilde{O}(d^3)$ time. The time to integrate dominates the complexity, and the overall complexity can be upper bounded by $O(T^2K\cdot d^4/(\eta/T)^2)=O(KT^4d^4)$.
Note: The approximate integration and sampling are only needed for multi-dimensional case, for the one-dimensional case we can compute the weights and sample exactly in polynomial time.
\end{proof}

\section{Lower bounds}\label{appendix:lb}
We start with a simple lower bound argument for $s$-shifted regret for prediction with two experts based on a well-known $\Omega(\sqrt{T})$ lower bound argument for static regret. We will then extend it to the continuous setting and use it for the $\Omega(\sqrt{sT})$ part of the lower bound in Theorem \ref{thm:lb} in Section \ref{sec:lb}.\\

\begin{lemma}
\label{lem:2-exp-lb}
For prediction with two experts, there exists a stochastic sequence of losses for which the $s$-shifted regret of any online learning algorithm satisfies
\[\E[R_T]\ge\sqrt{sT/8}\]
\end{lemma}
\begin{proof}
Let the two experts predict $0$ and $1$ respectively at each time $t\in[T]$. The utility at each time $t$ is computed by flipping a coin - with probability $1/2$ we have $u(0)=1,u(1)=0$ and with probability $1/2$ it's $u(0)=0,u(1)=1$. Expected payoff for any algorithm $\A$ is
\[P(\A,T)=\E\Big[\sum_{t=1}^Tu_t(\rho_t)\Big]=\sum_{t=1}^T\E[u_t(\rho_t)]=\frac{T}{2}\]
since expected payoff is $1/2$ at each $t$ no matter which expert is picked.\\
To compute shifted regret we need to compare this payoff with the best sequence of experts with $s-1$ switches. We compare with a weaker adversary $\A'$ which is only allowed to switch \textit{up to} $s-1$ times, and switches at only a subset of fixed times $t_i=iT/s$ to lower bound the regret.
\begin{align*}\E[R_T]&=OPT-P(\A,T) \\&\ge P(\A',T)-P(\A,T)\\&=\sum_{t=1}^T\E[u_t(\rho^{\prime}_t)]-\sum_{t=1}^T\E[u_t(\rho_t)]\\&=\sum_{i=0}^{s-1}\sum_{t=t_i+1}^{t_{i+1}}\E[u_t(\rho^\prime_t)]-\E[u_t(\rho_t)]\end{align*}
Now let $P_{i,j}=\sum_{t=t_i+1}^{t_{i+1}}\E[u_t(j)]$ for $i+1\in[s]$ and $j\in\{0,1\}$
\[\sum_{t=t_i+1}^{t_{i+1}}\E[u_t(\rho^\prime_t)]=\max_{\rho\in\{0,1\}}\sum_{t=t_i+1}^{t_{i+1}}\E[u_t(\rho)] =\frac{1}{2}\Big[P_{i,0}+P_{i,1}+|P_{i,0}-P_{i,1}|\Big]=\frac{T}{2s}+|P_{i,0}-T/2s| \]
using $P_{i,0}+P_{i,1}=T/s$. Thus,
\[\E[R_T]\ge \sum_{i=0}^{s-1}\Big[\Big(\frac{T}{2s}+|P_{i,0}-T/2s|\Big)-\frac{T}{2s}\Big]=\sum_{i=0}^{s-1}|P_{i,0}-T/2s|\]
Noting $P_{i,0}=\sum_{t=t_i+1}^{t_{i+1}}\E[u_t(0)]=\sum_{t=t_i+1}^{t_{i+1}}\Big(\frac{1+\sigma_t}{2}\Big)$ where $\sigma_t$ are Rademacher variables over $\{-1,1\}$ and applying Khintchine's inequality (see for example \cite{ben2009agnostic}) we get
\[\E[R_T]\ge \sum_{i=0}^{s-1}\biggl\lvert\sum_{t=t_i+1}^{t_{i+1}}\sigma_t/2\biggr\rvert\ge \sum_{i=0}^{s-1}\sqrt{T/8s}=\sqrt{sT/8}\]
\end{proof}

\begin{corollary}\label{cor:cont-lb}
We can embed the two-expert setting to get a lower bound for the continuous case.
\end{corollary}
\begin{proof}
 Indeed in Lemma \ref{lem:2-exp-lb} let $\C=[0,1]$, expert $0$ correspond to $\rho_1=1/4$, expert $1$ corresponds to $\rho_2=3/4$ and replace the loss functions by
\begin{align*}
    u^{(0)}(\rho)=
    \begin{cases*}
    1&if $\rho < \frac{1}{2}$\\
    0&if $\rho \ge \frac{1}{2}$
    \end{cases*}\text{ and }u^{(1)}(\rho)=
    \begin{cases*}
    0&if $\rho < \frac{1}{2}$\\
    1&if $\rho \ge \frac{1}{2}$
    \end{cases*}
\end{align*}
We can further generalize this while dispersing the discontinuities somewhat. Instead of having all the discontinuties at $\rho=\frac{1}{2}$, we can have discontinuities dispersed say within an interval $[\frac{1}{3},\frac{2}{3}]$ and still have $\Omega(\sqrt{sT})$ regret.
\end{proof}

\thmlb*
\begin{proof}[Proof of Theorem \ref{thm:lb}] $I_1=[0,1]$. In the first phase, for the first $\frac{T-3sT^{1-\beta}}{s}$ functions we have a single discontinuity in the interval $\left(\frac{1}{2}\left(1-\frac{1}{3s}\right),\frac{1}{2}\left(1+\frac{1}{3s}\right)\right)\subseteq(\frac{1}{3},\frac{2}{3})$. The functions have payoff 1 before or after (with probability $1/2$ each) their discontinuity point, and zero elsewhere. We introduce $3T^{1-\beta}$ functions each for the same discontinuity point, and set the discontinuity points $T^{-\beta}$ apart for $\beta$-dispersion. This gives us $\frac{1/3s}{T^{-\beta}}-1$ potential points inside $[\frac{1}{3},\frac{2}{3}]$, so we can support $3T^{1-\beta}\left(\frac{1/3s}{T^{-\beta}}-1\right)=\frac{T}{s}-3T^{1-\beta}$ such functions ($\frac{T}{s}-3T^{1-\beta}>0$ since $\beta>\frac{\log 3s}{\log T}$). By Lemma \ref{lem:2-exp-lb} we accumulate   $\Omega(\sqrt{\frac{T-3sT^{1-\beta}}{s}})=\Omega(\sqrt{T/s})$ regret for this part of the phase in expectation. Let $I_1'$ be the interval from among $[0,\frac{1}{2}\left(1-\frac{1}{3s}\right)]$ and $[\frac{1}{2}\left(1+\frac{1}{3s}\right),1]$ with more payoff in the phase so far. The next function has payoff 1 only at first or second half of $I_1'$ (with probability $1/2$) and zero everywhere else. Any algorithm accumulates expected regret $1/2$ on this round. We repeat this in successively halved intervals. $\beta$-dispersion is satisfied since we use only $\Theta(T^{1-\beta})$ functions in the interval $I'$ of size greater than $1/3$, and we accumulate an additional $\Omega(T^{1-\beta})$ regret. Notice there is a fixed point used by the optimal adversary for this phase.\\
Finally we repeat the construction inside the largest interval with no discontinuities at the end of the last phase for the next phase. Note that at the $i$-th phase the interval size will be $\Theta(\frac{1}{i})$. Indeed at the end of the first round we have {\it unused} intervals of size $\frac{1}{2}\left(1-\frac{1}{3s}\right),\frac{1}{4}\left(1-\frac{1}{3s}\right),\frac{1}{8}\left(1-\frac{1}{3s}\right),\dots$ At the $i=2^j$-th phase, we'll be repeating inside an interval of size $\frac{1}{2^{j+1}}\left(1-\frac{1}{3s}\right)=\Theta(\frac{1}{i})$. This allows us to run $\Theta(s)$ phases and get the desired lower bound (the intervals must be of size at least $\frac{1}{s}$ to support the construction). 
\end{proof}

\section{Experiments}\label{appendix:exp}
We supplement our results in Section \ref{sec:exp} by looking at different changing environments and comparing with performance in the static environment setting. We also look at differences between Generalized and Fixed Share EFs.

\subsection{Frequently changing environments}\label{sec:exp-freq}

In Section \ref{sec:exp} we presented a comparison of our algorithms Fixed Share EF (Algorithm \ref{alg:fsef}) and Generalized Share EF (Algorithm \ref{alg:gsef}) with the Exponential Forecaster algorithm of \cite{balcan2018dispersion} for online clustering using well-known datasets. We evaluated the $2$-shifted regret for problems where the clustering instance distribution changed exactly once and completely at $T/2$. Here we consider experiments with environments that change more gradually but more frequently.

We consider a sequence of clustering instances drawn from the four datasets. At each time $t \le T \le 50$ we sample a subset of the dataset of size $100$. For each $T$, we take uniformly random points from all but one classes. The omitted class is changed every $T/k$ rounds, where $k$ is the total number of classes for the dataset. We use parameters $\alpha=\frac{k-1}{T},\gamma=\frac{1}{T}$ in our algorithms. 
We compute the average regret against the best offline algorithm with $k$ shifts. In Figure \ref{fig:2} we plot the average of 20 runs for each dataset. The average regret is higher for all algorithms here since the $k$-shifted baseline is stronger.

\begin{figure}[ht]
    \centering
    \begin{subfigure}[b]{0.32\textwidth}
    \centering
         \includegraphics[width=\textwidth]{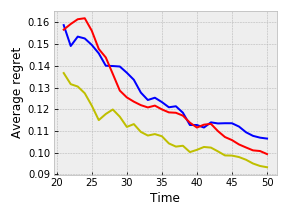}
    \caption{MNIST}
  \end{subfigure}
    \begin{subfigure}[b]{0.32\textwidth}
    \centering
         \includegraphics[width=\textwidth]{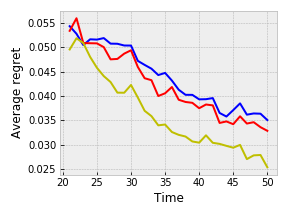}
    \caption{Omniglot\_small\_1}
  \end{subfigure}
    \begin{subfigure}[b]{0.32\textwidth}
    \centering
         \includegraphics[width=\textwidth]{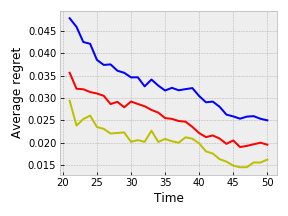}
    \caption{Omniglot (full)}
  \end{subfigure}
  \caption{Average $k$-shifted regret vs game duration $T$ for online clustering against $k$-shifted distributions. Color scheme: \textcolor{blue}{\bf Exponential Forecaster}, \textcolor{red}{\bf Fixed Share EF}, \textcolor{colorY}{\bf Generalized Share EF}}
\label{fig:2}
\end{figure}

\subsection{Generalized vs Fixed Share EFs}

\begin{figure}[ht]
    \centering
    \begin{subfigure}[b]{0.32\textwidth}
    \centering
         \includegraphics[width=\textwidth]{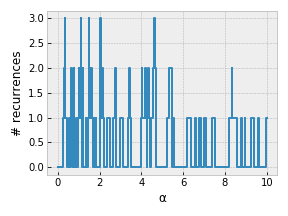}
    \caption{MNIST}
  \end{subfigure}
    \begin{subfigure}[b]{0.32\textwidth}
    \centering
         \includegraphics[width=\textwidth]{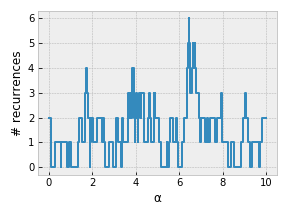}
    \caption{Omniglot\_small\_1}
  \end{subfigure}
    \begin{subfigure}[b]{0.32\textwidth}
    \centering
         \includegraphics[width=\textwidth]{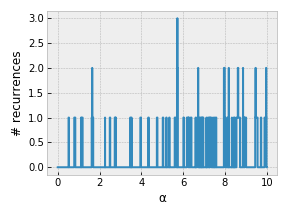}
    \caption{Omniglot (full)}
  \end{subfigure}
  \caption{Number of recurrences of various values of $\alpha$ in the top decile across all rounds}
\label{fig:22}
\end{figure}
We note that Generalized Share EF performs better on most problem instances. This is because it is better able to use recurring patterns in {\it good} values for the parameter that occur non-contiguously, which depends upon the dataset and the problem instance. We verify this hypothesis by a simple experiment.

We compute the set of intervals containing the top 10\% of the measure of $\alpha\in[0,10]$ for each $t$ and sum up occurrences of such intervals across all rounds. We observe most recurrences in {\it Omniglot\_small\_1} dataset, which explains the large gap between Generalized vs Fixed Share EFs.

\subsection{Comparison with static environments}
We compare the performance of Fixed Share EF with Exponential Forecaster in static vs dynamic environments on the MNIST dataset. For the changing environment we consider the setting of Section \ref{sec:exp}, where we present clustering instances for even digits for $t=1$ through $t=T/2$ and odd digits thereafter. For the static environment we continue to present clustering instances from even labeled digits even after $t=T/2$. We plot the $2$-shifted regret in both cases for easier comparison (Figure \ref{fig:3}). Note that even though static regret is the more meaningful metric in a static environment, this only changes the baseline and the relative performance of algorithms is unaffected by this choice.

Notice that Fixed Share EF is slightly better in the static environment but significantly better in the dynamic environment. It's also worthwhile to note that while the performance of Exponential Forecaster degrades with changing environment, Fixed Share EF actually improves in the dynamic environment since the exploratory updates are more useful.

\begin{figure}[ht]
    \centering
    \begin{subfigure}[b]{0.32\textwidth}
    \centering
         \includegraphics[width=\textwidth]{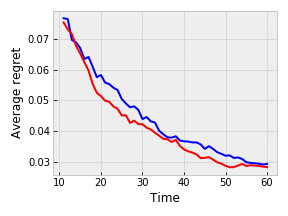}
    \caption{static environment}
  \end{subfigure}
    \begin{subfigure}[b]{0.32\textwidth}
    \centering
         \includegraphics[width=\textwidth]{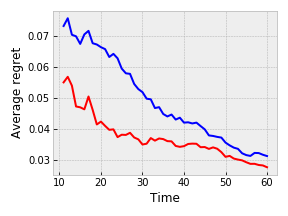}
    \caption{dynamic environment}
  \end{subfigure}
  \caption{Average 2-shifted regret vs game duration $T$ for online clustering against static/dynamic distributions for the MNIST dataset. Color scheme: \textcolor{blue}{\bf Exponential Forecaster}, \textcolor{red}{\bf Fixed Share EF}}
\label{fig:3}
\end{figure}

\subsection{Different environments from the same dataset}
We look at 2-shifted regret of MNIST clustering instances with the same setting as in Section \ref{sec:exp} but with different partitions of clustering classes (i.e. classes used before and after $T/2$). The results are summarized in Figure \ref{fig:4}. For each instance we note the set of 5 digits used for drawing uniformly random clustering instances from MNIST till $T/2$, the complement set is used for the remaining rounds. We observe that performance gap between Fixed Share EF and Exponential Forecaster depends not only on the dataset, but also on the clustering instance from the dataset. Across several partitions, Fixed Share EF performs significantly better on average (Figure \ref{fig:4} (f)).

\begin{figure}[ht]
    \centering
    \begin{subfigure}[b]{0.32\textwidth}
    \centering
         \includegraphics[width=\textwidth]{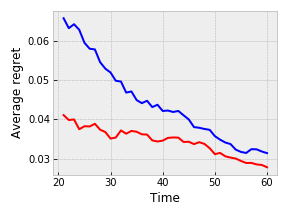}
    \caption{$\{0,2,4,6,8\}$}
  \end{subfigure}
    \begin{subfigure}[b]{0.32\textwidth}
    \centering
         \includegraphics[width=\textwidth]{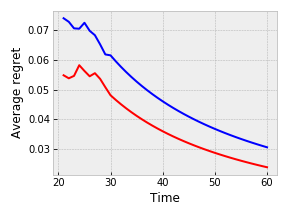}
    \caption{$\{0,1,2,3,4\}$}
  \end{subfigure}
    \begin{subfigure}[b]{0.32\textwidth}
    \centering
         \includegraphics[width=\textwidth]{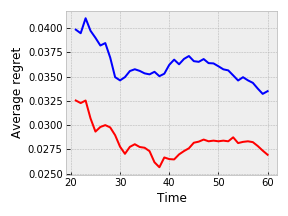}
    \caption{$\{2,3,5,6,9\}$}
  \end{subfigure} \\
    \begin{subfigure}[b]{0.32\textwidth}
    \centering
         \includegraphics[width=\textwidth]{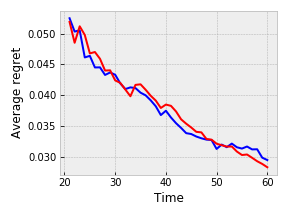}
    \caption{$\{1,3,4,8,9\}$}
  \end{subfigure}
    \begin{subfigure}[b]{0.32\textwidth}
    \centering
         \includegraphics[width=\textwidth]{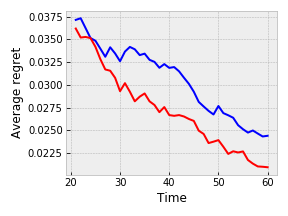}
    \caption{$\{0,4,5,7,8\}$}
  \end{subfigure}
    \begin{subfigure}[b]{0.32\textwidth}
    \centering
         \includegraphics[width=\textwidth]{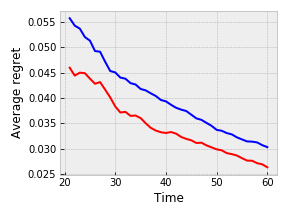}
    \caption{Average}
  \end{subfigure} 
  \caption{Average 2-shifted regret vs game duration $T$ for online clustering against various dynamic instances for the MNIST dataset. Color scheme: \textcolor{blue}{\bf Exponential Forecaster}, \textcolor{red}{\bf Fixed Share EF}}
\label{fig:4}
\end{figure}

\end{document}